\documentclass[letterpaper]{article} % DO NOT CHANGE THIS
\usepackage{aaai24}  % DO NOT CHANGE THIS
\usepackage{times}  % DO NOT CHANGE THIS
\usepackage{helvet}  % DO NOT CHANGE THIS
\usepackage{courier}  % DO NOT CHANGE THIS
\usepackage[hyphens]{url}  % DO NOT CHANGE THIS
\usepackage{graphicx} % DO NOT CHANGE THIS
\usepackage{natbib}  % DO NOT CHANGE THIS AND DO NOT ADD ANY OPTIONS TO IT
\usepackage{caption} % DO NOT CHANGE THIS AND DO NOT ADD ANY OPTIONS TO IT
\usepackage{algorithm}
\usepackage{algorithmic}

\usepackage{newfloat}
\usepackage{listings}
\DeclareCaptionStyle{ruled}{labelfont=normalfont,labelsep=colon,strut=off} % DO NOT CHANGE THIS
\floatstyle{ruled}
\newfloat{listing}{tb}{lst}{}
\floatname{listing}{Listing}
\usepackage{amsmath}
\usepackage{amssymb}
\usepackage{amsthm}
\usepackage{bm}
\usepackage[capitalise]{cleveref}

\usepackage{caption}
\usepackage{subcaption}

\newtheorem{theorem}{\bf Theorem}
\newtheorem{lemma}{\bf Lemma}

\newtheorem{assumption}{\bf Assumption}

\usepackage{threeparttable, array} % for adding note at the end of a table
\usepackage{pifont}% http://ctan.org/pkg/pifont

\DeclareMathOperator*{\argmax}{argmax}

\newcommand{\norm}[1]{\left\lVert#1\right\rVert}

\newcommand{\inner}[1]{\langle#1\rangle}
\newcommand{\cS}{\mathcal{S}}

\newcommand{\R}{\mathbb{R}}

\newcommand{\E}{\mathbb{E}}

\newcommand{\I}{\mathbb{I}}

\newcommand{\ba}{\boldsymbol{a}}

\newcommand{\bM}{\boldsymbol{M}}

\newcommand{\bw}{\boldsymbol{w}}

\newcommand{\bx}{\boldsymbol{x}}

\newcommand{\btheta}{\boldsymbol{\theta}}

\newcommand{\bSigma}{\boldsymbol{\Sigma}}

\newcommand{\cA}{\mathcal{A}}

\newcommand{\cH}{\mathcal{H}}
\newcommand{\cI}{\mathcal{I}}

\newcommand{\cU}{\mathcal{U}}
\newcommand{\cV}{\mathcal{V}}

\newcommand{\gtclusters}{\textit{ground-truth clusters}}

\title{Federated Contextual Cascading Bandits with Asynchronous Communication and Heterogeneous Users}

\author {
    Hantao Yang\textsuperscript{\rm 1},
    Xutong Liu\textsuperscript{\rm 2}\thanks{Corresponding Author},
    Zhiyong Wang\textsuperscript{\rm 2},
	Hong Xie\textsuperscript{\rm 1},
	John C. S. Lui\textsuperscript{\rm 2},
	Defu Lian\textsuperscript{\rm 1},
	Enhong Chen\textsuperscript{\rm 1}
}
\affiliations {
    \textsuperscript{\rm 1}University of Science and Technology of China\\
    \textsuperscript{\rm 2}The Chinese University of Hong Kong\\
    yanghantao@mail.ustc.edu.cn, liuxt@cse.cuhk.edu.hk, zywang21@cse.cuhk.edu.hk, xiehong2018@foxmail.com, cslui@cse.cuhk.edu.hk, liandefu@ustc.edu.cn, cheneh@ustc.edu.cn
}

\begin{document}

\maketitle

\begin{abstract}
We study the problem of federated contextual combinatorial cascading bandits, where $|\mathcal{U}|$ agents collaborate under the coordination of a central server to provide tailored recommendations to the $|\mathcal{U}|$ corresponding users. Existing works consider either a synchronous framework, necessitating full agent participation and global synchronization, or assume user homogeneity with identical behaviors. We overcome these limitations by considering (1) federated agents operating in an asynchronous communication paradigm, where no mandatory synchronization is required and all agents communicate independently with the server, (2) heterogeneous user behaviors, where users can be stratified into $J \le |\mathcal{U}|$ latent user clusters, each exhibiting distinct preferences. For this setting, we propose a UCB-type algorithm with delicate communication protocols. Through theoretical analysis, we give sub-linear regret bounds on par with those achieved in the synchronous framework, while incurring only logarithmic communication costs. Empirical evaluation on synthetic and real-world datasets validates our algorithm's superior performance in terms of regrets and communication costs.

\end{abstract}

\section{Introduction}

Contemporary recommendation systems commonly employ ordered lists to present recommended items, which cover applications in diverse domains like news, restaurants, and movies recommendations, as well as search engines \cite{lian2023training,wu2023influence}.  
Recent attention has been directed towards the analysis of user interaction patterns and feedback in such ordered lists, 
leading to the development of a \textit{cascade model} by \citet{craswell2008experimental}. 
In this model, users engage with a list of items sequentially, clicking on the first satisfactory item encountered.  
Following a click, users cease further examination of the list. Learning agent receives feedback which indicates that the items before the clicked one were examined and deemed unsatisfactory, while the user's preference of items \textit{after} the clicked entry remains unknown.  
Despite its apparent simplicity, cascade model has demonstrated efficacy in capturing user behaviors, as highlighted by \citet{chuklin2015click}.

Our study focuses on an online learning variant of the cascade model, 
referred to as \textit{cascading bandits} (CB) \cite{kveton2015cascading, kveton2015combinatorial}. In the CB framework, the learning agent employs 
exploration-exploitation techniques to comprehend user preferences 
over items through interactions.  
Specifically, in each round, the agent recommends an item list, observes click feedback, and receives a reward of 1 if the user clicks on any item (otherwise, a reward of $0$ is received).  
The agent's goal is to maximize cumulative rewards over $T$ rounds (or equivalently minimize the $T$-round regret).

Motivated by large-scale applications with a huge number of items, previous studies \cite{li2016contextual,zong2016cascading} have proposed the \textit{contextual combinatorial cascading bandits} (C$^3$B) to integrate contextual information. This integration is achieved by
adding a linear structure assumption, enhancing the scalability of CB. Later on, \citet{li2018online} extends C$^3$B to accommodate the heterogeneity of users' preferences for items. By assuming that users can be partitioned into $J$ unknown user groups (i.e., clusters), they introduce the CLUB-cascade algorithm, where agents adaptively cluster users into groups, and harness the collaborative influence of these user groups to enhance the performance. 
%of recommended lists.

While the successes of C$^3$B in handling contextual information and user heterogeneity are noteworthy, prior investigations have primarily focused on either the single-agent paradigm \cite{li2016contextual,zong2016cascading} or the synchronous setting which requires simultaneous data synchronization among agents \cite{li2018online}. In real-world scenarios, however, there could be a large number of participating agents, ranging from powerful servers to resource-constrained mobile devices. Furthermore, users can be highly heterogeneous in terms of preferences and arrival times. In this case, the feasibility of managing synchronous communication from all agents becomes questionable. This motivates us to study C$^3$B within an \textit{asynchronous communication framework}, which can accommodate a large number of agents without requiring them to communicate at the same time, and allow agents to serve users with varying time of arrival and heterogenous preferences.

\begin{table}[t]
\small
\centering
	\begin{threeparttable}
    \begin{tabular}{c|c|c|c}
    \hline
        \textbf{Algorithm}&\textbf{Agent}& \textbf{User} & \textbf{Arm}\\
    \hline
     OFUL \shortcite{abbasi2011improved}& Single   & Hom & Single \\
     C$^3$-UCB \shortcite{li2016contextual} & Single & Hom & Cascade \\     
	\hline
     DisLinUCB \shortcite{2019Distributed}  &   Syn & Hom & Single \\
     CLUB-cascade \shortcite{li2018online}  &   Syn & Heter & Cascade\\ 
	\hline
     Async-LinUCB \shortcite{li2021asynchronous} &   Par-Asy & Hom & Single\\
     Async-LinUCB-AM \shortcite{li2021asynchronous}   &   Par-Asy & Heter & Single\\
     Phase-based FCLUB \shortcite{liu2022federated}  &  Par-Asy & Heter & Single\\ 
	\hline
	FedLinUCB \shortcite{he2022simple} &   Asy & Hom & Single\\
    FedC$^3$UCB-H  (Ours) & Asy    & Heter & Cascade\\
	\hline
    \end{tabular}
	\begin{tablenotes}[para, online,flushleft]
	\footnotesize
	%\item[*] s
	\end{tablenotes}
	\end{threeparttable}
\caption{Compared with existing works regarding agent communication (`Syn' is synchronous, `Asy' is asynchronous and `Par' is partially), user model (`Hom' is homogeneous and `Heter' is heterogeneous), and arm selection.
\label{tab:fed_res1}
}	
\end{table}

Our work introduces the first federated contextual combinatorial bandits (FedC$^3$B), a framework wherein a finite set of users
$\mathcal{U}$ are served by $|\mathcal{U}|$ learning agents.  
Each agent is responsible for offering personalized recommendations to a specific user, and collaboration is facilitated through a central server.
Due to the stochastic nature of user arrivals, communication between agents occurs asynchronously, eliminating the need for costly synchronizations and allowing independent communication with the central server.
To account for user heterogeneity, we assume users can be categorized into $J \le |\mathcal{U}|$ unknown clusters, each representing a group of users with similar user behaviors. 
The key challenge, therefore, is to  
accurately and swiftly determine these user clusters to facilitate effective cooperation among agents with similar users, while minimizing cross-contamination of information arising from users with different behaviors. To make the problem even more challenging, asynchronous communication introduces the additional complexity of potential delays and data inconsistencies, which shall be carefully handled to give meaningful regret/communication guarantees.
\subsection{Our Contributions}
To address the aforementioned challenges, this paper makes three key contributions as follows.

\noindent\textbf{Algorithmic Framework.}
We propose the Federated Contextual Combinatorial 
Cascading Upper Confidence Bound Algorithm for Heterogeneous Users (FedC$^3$UCB-H). FedC$^3$UCB-H contains three components: agent-side action selection based on the information received asynchronously, server-side graph-based heterogeneity testing for precise user clustering, and agent-server asynchronous communication strategies to manage potential data inconsistencies. The key challenge of the algorithm design is to manage the data inconsistency during user clustering. Existing matrix determinant-based protocols ~\cite{li2021asynchronous,liu2022federated,he2022simple} fail to achieve this goal due to insufficient communication, so we propose a novel communication protocol that employs a $p_t$-auxiliary protocol in conjunction with the matrix determinant-based protocol. This new protocol effectively controls data inconsistency, ensuring the correct operation of heterogeneity testing and action selection. 
%components.

\noindent\textbf{Theoretical Analysis.} 
We prove that FedC$^3$UCB-H achieves a $O(d\sqrt{JKT}\log T)$ regret bound with a communication cost of $O(d|\cU|\log T + \log^2 T)$ communication cost, showing that we use only logarithmic communication cost to achieve regret results on par with the synchronous paradigm that incurs $O(T)$ communication cost \cite{li2018online}. Moreover, our result generalizes the federated linear contextual bandits \cite{he2022simple} with homogeneous users by allowing $J>1$, while matching their regret and communication cost when $J=1$. Our analysis tackles several technical challenges, such as showing that within a short time span and at a low cost of asynchronous communication, each agent will collaborate with the correct counterparts, and managing information gaps for asynchronous collaboration involving multiple arms (cascading arms). 
We believe our proof techniques are novel and may be of independent interest for analyzing related works that involve asynchronous communication, heterogeneous users, and cascading arms.

\noindent\textbf{Empirical Evaluation.}
Finally, we conduct experiments on both synthetic and real data and show the effectiveness of the user clustering procedure and the $p_t$-auxiliary communication, where our FedC$^3$UCB-H achieves superior performance regarding regrets and communication cost.

\begin{table}[t]
\small
\centering
    \begin{tabular}{c|c}
    \hline
         \textbf{Regret} &\textbf{Communication}\\
    \hline
      $O(d\sqrt{T}\log{T})$ & N/A \\
       $O(d\sqrt{KT}\log{T})$ & N/A \\     
	\hline
       $O(d\sqrt{T}\log^2{T})$ & $O(d|\mathcal{U}|^{1.5})$ \\
      $O(d\sqrt{JKT}\log{T})$ & $O(T)$\\ 
	\hline
      $O(d\sqrt{T}\log{T})$ & $O(d|\mathcal{U}|^2\log{T})$\\
     \ $O(d\sqrt{|\mathcal{U}|T}\log{T})$ & $O(d|\mathcal{U}|^2\log{T})$\\
      $O(d|\mathcal{U}|\sqrt{JT}\log^{1.5}{T})$ & $O(d|\mathcal{U}|J\log{T})$\\ 
	\hline
	 $O(d\sqrt{T}\log{T})$ & $O(d|\mathcal{U}|^2\log{T})$\\
     $O(d\sqrt{JKT}\log{T})$ & $O(d|\mathcal{U}|^2\log{T}+\log^2{T})$\\
	\hline
    \end{tabular}
\caption{Each row of \cref{tab:fed_res2} is related to that of \cref{tab:fed_res1}. Our framework generalizes existing works and achieves near-optimal regret with low communication cost.}\label{tab:fed_res2}	
\end{table}

\section{Related Work}

We review and compare with existing algorithms in the area of  contextual linear bandits. We consider three primary dimensions (1) cascading bandit frameworks, (2) user heterogeneity considerations, and (3) distributed/federated agents. Detailed comparisons are summarized in \Cref{tab:fed_res1,tab:fed_res2}. 

\noindent\textbf{Cascading Bandit.} Cascading bandits (CB) and their variants belong to the field of online learning to rank, which has a vast literature and covers a wide range of applications \cite{chuklin2015click,kveton2015cascading,kveton2015combinatorial,li2016contextual,lian2020personalized,vial2022minimax,liu2022batch,liu2023contextual,liu2023variance,choi2023cascading}. CB is first proposed by \citet{kveton2015cascading} and then generalized by \citet{kveton2015combinatorial} to accommodate combinatorial action spaces. To handle large-scale applications, \citet{zong2016cascading} integrates the contextual information and introduces the contextual cascading bandits. Based on the seminal work of linear bandits \cite{abbasi2011improved} (row 1 in \Cref{tab:fed_res1}), which leverages the optimism in the face of uncertainty principle (OFUL), CascadeLinUCB is introduced, yielding a regret bound of $O(d\sqrt{T})$. Their results are generalized by \citet{li2016contextual} to accommodate combinatorial actions,  yielding a regret bound of $O(d\sqrt{KT})$ (row 2 in \Cref{tab:fed_res2}), and is improved to $O(d\sqrt{T})$ quite recently by  \cite{liu2023contextual} and \cite{choi2023cascading}. All these works deal with single-agent scenarios (where no communication is needed) and assume  homogeneous users, leading to techniques that are different and simpler than ours.

\noindent\textbf{Online Clustering of Bandits.} Online clustering of bandits (CLUB) have emerged as a prominent line of works to address user heterogeneity in contextual bandits. The first work is proposed by \citet{gentile2014online}, followed by an extension by \citet{li2018online}, wherein users pull cascading arms (row 4 in \Cref{tab:fed_res1}). Both works, including ours, adopt the key idea of maintaining a user heterogeneity graph to adaptively refine user clusters, resulting in regret bounds that depend on total user groups $J$ instead of individual users $|\cU|$. However, the above two works consider either single-agent setting, or synchronous setting that requires $O(T)$ commmunication to obtain up-to-date user information for decision-making. Our work, our work deviates by focusing on asynchronous multi-agent settings, where managing potential information delays and data inconsistencies are pivotal considerations.

\noindent\textbf{Federated/Distributed Contextual Linear Bandits.} There has been growing interest in bandit learning with multiple agents. Our work belongs to the field of cooperative distributed/federated bandits, in which multiple agents collaborate to solve a linear contextual bandit problem over certain communication networks, e.g., peer-to-peer networks~\cite{korda2016distributed,zhu2021federated,xu2023decentralized} or star-shaped networks~\cite{dubey2020differentially,li2021asynchronous}. Our work aligns with the latter category, where a central server engages with multiple agents. For this setting, \citet{2019Distributed} proposes DisLinUCB algorithm and achieves near-optimal $O(d\sqrt{T})$ in the synchronous setting, where users/agents are homogeneous and mandatory synchronization is required (row 2 in \Cref{tab:fed_res1,tab:fed_res2}). \citet{li2021asynchronous} explores the partially asynchronous setting, where no synchronization is mandated, but agents still do not communicate with the server independently and one agent's upload will trigger other agents' download (row 5 in \Cref{tab:fed_res1}). Notably, their work also considers heterogeneous users but treats them as totally different entities, and thus achieves sub-optimal regret that has a $\sqrt{|\cU|}$ factor (row 6 in \Cref{tab:fed_res2}). Ours improves this to $\sqrt{J}$ regret factor. Later on, \citet{liu2022federated} considers the similarity between users by considering user groups, similar to our approach, yet their communication model remains partially asynchronous, and the regret bound has a $|\cU|$ factor (row 7 in \Cref{tab:fed_res2}). The recent FedLinUCB algorithm by \citet{he2022simple} for homogeneous users in a fully asynchronous setting  (row 8 in \Cref{tab:fed_res1}) provides inspiration for our work.  We generalize their approach to encompass heterogeneous users and cascading arms, yielding matching results when $J=1$ and $K=1$.

\section{Problem Setup}\label{sec:problem_setting}
\textbf{Notations.} Let $n\in \mathbb{N}_+$ be a positive integer. $[n]$ denotes the set $\{1,...,n\}$. For any set $\mathcal{S}$, $|\mathcal{S}|$ denotes the number of elements in $\cS$. For any event $\mathcal{E}$, we use $\I\{\mathcal{E}\}$ to denote the indicator function, where $\I\{\mathcal{E}\}=1$ if $\mathcal{E}$ holds and $\I\{\mathcal{E}\}=0$ otherwise. We use boldface lowercase letters and boldface capitalized letters to represent column vectors and matrices, respectively. For vector norms, $\norm{\bx}_p$ denotes the $\ell_p$ norm of vector $\bx$. For any symmetric positive semi-definite (PSD) matrix $\bM$ (i.e., $\bx^{\top} \bM \bx \ge 0, \forall \bx$), $\norm{\bx}_{\bM}=\sqrt{\bx^{\top} \bM \bx}$ denotes the matrix norm of $\bx$ regarding matrix $\bM$.

\noindent\textbf{Federated Contextual Cascading Bandits.}
In this section, we formulate the setting of ``Federated Contextual Combinatorial Cascading Bandits'' (FedC$^3$B). 
We consider a finite set 
$\mathcal{I} \triangleq \{1,...,I\}$ of 
 \textit{ground items} to be selected with $I\in \mathbb{N}_+$, referred to as \textit{base arms}.
Let $\Pi(\mathcal{I})=\{(a_1, ..., a_k): k \ge 1, a_1,..., a_k \in \mathcal{I}, a_i\neq a_j \text{ for any }i \neq j\}$ be the set of all possible tuples of distinct items selected from $\mathcal{I}$, where we refer each of such tuples as an \textit{action}. For any action $\ba \in \cA$, it involves at most $K$ items, i.e., $\text{len}(\ba)\le K$, where $K\le I$.

In FedC$^3$B, there is a finite set $\mathcal{U}$ 
of users, to be served by $|\mathcal{U}|$ (local) learning agents individually. 
Each learning agent serves one user $u$ in order to provide 
personalized service to that user.\footnote{We also use $u$ to identify the 
agent who serves user $u$. } 
Each user $u\in \mathcal{U}$ is associated with an \textit{unknown} preference vector $\btheta_u\in\mathbb{R}^d$, with $\norm{\btheta_u}_2\leq 1$, 
where $d\in \mathbb{N}_+$. 
To characterize the heterogeneity inherent in the user population, 
we assume users are partitioned into disjoint clusters, and two users share the 
same preference vector, if and only if they belong to the same cluster. 
Specifically, let 
$\mathcal{V}_1,\ldots \mathcal{V}_J$ 
denote the partitioned clusters, 
where $J \in \mathbb{N}_+, J \le |\mathcal{U}|$, 
$\cup_{j \in [J]} \mathcal{V}_j=\mathcal{U}$ 
and $\mathcal{V}_j \cap \mathcal{V}_{j'}=\emptyset,$ 
for $j\neq j'$.  
Moreover, 
$\bm{\theta}_u = \bm{\theta}_{u'}$ if and only if 
there exists $j\in[J]$ such that  
$\{u, u'\} \subseteq \mathcal{V}_j$. 
For the ease of presentation, 
let $\btheta^j$ denote the shared preference vector for cluster $\mathcal{V}_j$.  
These clusters, termed \gtclusters{}, remain \textit{unknown} to the agent. 
Importantly, in the scenario where $J=1$, it is implied that $\mathcal{U}$ constitutes homogeneous users.

We consider a total number of $T\in \mathbb{N}_+$ 
learning rounds.  
In each round $t \in [T]$, a single user identified as $u_t$ arrives
to receive service. The user $u_t$ is presented with a finite set of items $\mathcal{I}_t \subseteq 
\mathcal{I}$ and feasible action set $A_t \subseteq \Pi(\mathcal{I}_t)$. 
Let $\bm{x}_{t,i} \in \mathbb{R}^d$ 
denote the feature vector of 
item $i \in \mathcal{I}_t$, 
where $\norm{\bx_{t,i}}_2\le 1$, which are 
revealed to the agent $u_t$ responsible for user $u_t$.  
The agent serving $u_t$ takes a feasible action  
$\bm{a}_t=(a_{1,t},...,a_{\texttt{len}(\bm{a_t}),t}) \in \mathcal{A}_t$, i.e., recommends 
a list of items $\bm{a}_t$ to the user.\footnote{We omit  superscript $u_t$ for $\bm{a}_t^{(u_t)}$, etc., when contexts are clear.} 
The user $u_t$ checks these items $\ba_t$ in sequence, beginning with the first, clicking on the first attractive item, and stopping checking items after the clicked item. 
We use the Bernoulli random variable $w_{t}(i)\in \{0,1\}$ to indicate whether item $i \in \mathcal{I}_t$
would be clicked or not once checked. The learning agent observes the index of the first clicked item or $+\infty$ (if no item is clicked), defined as
\begin{align}\label{eq:feedback_model}
O_t =\min\{1\le k \le \texttt{len}(\bm{a}_t): 
w_{t}(a_{k,t})=1\}
\end{align}
Here $O_t$ fully determines $w_{t}(a_{k,t})=1-\I\{k < O_t \}$ for $k= 1, ..., \min\{O_t, 
\texttt{len}(\bm{a}_t)\}$. 
Accordingly, we say that item $i$ is observed in round $t$, 
if $O_t < \infty$ and $a_{t,O_t}=i$. 
Since the user clicks are shaped by their personal inclinations, the random vector  
$( w_{t}(i) )_{i \in \mathcal{I}_t}$ is presumed to be independent of users other than $u_t$. 
We assume $w_t(i)$'s across different items are conditionally independent with expectation
\begin{align}
\bar{w}_{t}(i)\triangleq\E[w_t(i)\mid\cH_t^{(u_t)}]=\inner{\btheta_{u_t}, \bm{x}_{t,i}}.
\end{align}
where $\cH^{(u_t)}_{t} $ is the history 
associated with user $u_t$: $\cH^{(u_t)}_{t} \triangleq \{\bx_{t,i}\}_{i \in \mathcal{I}_t}\cup \{\{\bx_{s,i}\}_{i \in \mathcal{I}_t}, \bm{a}_s, O_s, w_s(a_{k,s})\}_{k \le O_s, s<t, u_s=u_t}$.
%We assume the clusters, users, and items satisfy the following assumptions.
We make the following assumptions.
\begin{assumption}[Minimum heterogeneity gap]
\label{assumption1}
The gap between any two preference vectors for different \gtclusters{} is at least $\gamma$, $\|\btheta^{j}-\btheta^{j^{\prime}}\|_2\geq \gamma, \forall{j,j^{\prime}\in [J]\,, j\neq j^{\prime}}\,,$
% \begin{equation*}
%     \norm{\btheta^{j}-\btheta^{j^{\prime}}}_2\geq \gamma>0\,, \forall{j,j^{\prime}\in [J]\,, j\neq j^{\prime}}\,,
% \end{equation*}
where $\gamma>0$ is an \textit{unknown} positive constant.
\end{assumption}

\begin{assumption}[Random arrival of users]
\label{assumption2}
At each round $t$, a user $u_t$ comes uniformly at random from $\mathcal{U}$ with probability $1/|\mathcal{U}|$, independent of the past rounds.
% \footnote{For the case where users have different $p_u$, one can easily extend our result by replacing $1/|\mathcal{U}|$ with $p_{\min}=\min_{u \in \cU}p_u$ .}
\end{assumption}

\begin{assumption}[Item regularity]
\label{assumption3}
At each time step $t$, the feature vector $\bx_{t,a}$ of each arm $a\in E$ is drawn independently from a fixed but unknown distribution $\rho$ over $\{\bx\in\R^d:\norm{\bx}_2\leq1\}$, where $\E_{\bx\sim \rho}[\bx \bx^{\top}]$ is full rank with minimal eigenvalue $\lambda_x > 0$. Additionally, at any time $t$, for any fixed unit vector $\btheta \in \R^d$, $(\btheta^{\top}\bx)^2$ has sub-Gaussian tail with variance upper bounded by $\sigma^2$.
\end{assumption} 

\noindent\textbf{Remark. } All these assumptions align with previous works on CB  \cite{gentile2014online,
li2018online,
liu2022federated,wang2023online}.  Notably, Assumption \ref{assumption3} stands as a more practical and less restrictive alternative compared to previous CB works that impose restrictions on the variance upper bound $\sigma^2$. Assumption \ref{assumption2}, our results can easily generalize to arbitrary distributions with minimum arrival probability greater than $p_{\min}>0$.

\noindent\textbf{Learning Objective.}
At round $t$, the reward of each action 
$\bm{a} \in \mathcal{A}_t$ with random vector 
$\bw_t \triangleq (w_t(i))_{i \in \mathcal{I}_t}$ is 
defined as: 
\begin{align}\label{eq:reward_func}\textstyle
f(\bm{a}, \bw_t) 
\triangleq 
1 - \prod_{k=1}^{\texttt{len}(\bm{a})} (1-w_t(a_k)). 
\end{align}
Let $\cH_t=\bigcup_{u \in \mathcal{U}} \cH_t^{u}$ be the total historical information up to time $t$, then by independence assumption, it is easy to verify that the expected reward is 
    $\E[f(\bm{a}, \bw_t)|\cH_t]
    =
    f(\bm{a}, \bar{\bw}_t).$
Let $\bm{a}^\ast_t=\argmax_{\bm{a} \in \mathcal{A}_t} 
f(\bm{a}, \bar{\bw}_t)$ be the optimal action at $t$.
% ({ \color{blue} $\cH_t$ is undefined!!}). 
Given the complexity of computing the optimal action $\bm{a}^\ast_t$ for a general feasible action set $\mathcal{A}_t$, even when $\btheta_{u_t}$ is known, we introduce an offline $\alpha$-approximation oracle. This oracle, provided to the agent, generates an action $\widetilde{\bm{a}}$ for any weight vector $\bw \in \R^{I_t}$ such that $f(\widetilde{\bm{a}},\bw)\ge \alpha \max_{\bm{a} \in \mathcal{A}_t} f(\bm{a},\bw)$.
% Since we consider the general feasible action set $\mathcal{A}_t$, computing the exact $\bm{a}^\ast_t$ could be NP-hard even when $\btheta_{u_t}$ is known, we assume the agent has access to an offline $\alpha$-approximation oracle, which for any weights $\bw \in \R^{I_t}$ the oracle outputs an action $\widetilde{\bm{a}}$ such that $f(\widetilde{\bm{a}},\bw)\ge \alpha \max_{\bm{a} \in \mathcal{A}_t} f(\bm{a},\bw) $.
The goal of the agents is to collaboratively minimize the cumulative $\alpha$-approximate regret defined as
\begin{align}\textstyle
    \text{Reg}(T) 
    =
    \E\left[\sum_{t=1}^T 
    \left(\alpha f(\bm{a}^*_t, \bw_t)
    -f(\bm{a}_t, \bw_t)\right)\right],
\end{align}
where the expectation is taken over the user arrival, the contexts, the weights, and the algorithm itself. 

\noindent\textbf{Communication model.} We adopt the server-client communication protocol with $|\mathcal{U}|$ local agents and one central server, where each agent can communicate with the server by uploading and downloading data. Unlike peer-to-peer communication~\cite{korda2016distributed,zhu2021federated}, local agents do not communicate with each other directly. Moreover, we adopt the \textit{asynchronous communication} paradigm: 
(1) there is no mandatory synchronization, (2) the communication between an agent and the server operates independently of other agents, without triggering additional communication. 
Finally, we define the communication cost as 
$\text{Com}(T)=\E[\sum_{t=1}^T\I \{\text{agent } u_t \text{ communications with the server}\}]$.

\begin{table}[t]
\resizebox{0.97\columnwidth}{!}{
\begin{minipage}{\columnwidth}
\center
\begin{tabular}{clclclcl}
\hline\hline
\multicolumn{2}{c|}{Notation} & \multicolumn{6}{c}{Meaning}              \\ \hline
\multicolumn{2}{c|}{$\hat{\btheta}_{u,t},\bm{\Sigma}_{u,t}$} & \multicolumn{6}{c}{Data for agent $u$'s decision}            \\ \hline
\multicolumn{2}{c|}{$\bm{\Sigma}_{u,t}^{loc}, \bm{b}_{u,t}^{loc}, T_{u,t}^{loc}$} & \multicolumn{6}{c}{Local data of agent $u$}            \\ \hline
\multicolumn{2}{c|}{$\bm{\Sigma}_{u,t}^{ser}, \bm{b}_{u,t}^{ser}, T_{u,t}^{ser}$} & \multicolumn{6}{c}{Data stored at the server for agent $u$ }            \\ \hline
\multicolumn{2}{c|}{$\hat{\btheta}_{u,t}^{clu}, \bm{\Sigma}_{u,t}^{clu}, \bm{b}_{u,t}^{clu}$} & \multicolumn{6}{c}{Data of the cluster that contains $u$}             \\ \hline\hline
\end{tabular}
\caption{Notations used in Algorithm \ref{alg:main}, \ref{alg:local} 
 and \ref{alg:server}.}\label{tab:notation}
 \end{minipage}}
\end{table}

\begin{algorithm}[htb] 
\resizebox{0.97\columnwidth}{!}{
\begin{minipage}{\columnwidth}
    \caption{FedC$^3$UCB-H}\label{alg:main} 
    \label{alg:Framwork} 
    \begin{algorithmic}[1] 
    \STATE \textbf{Input:} Communication and deletion thresholds $\alpha_c, \alpha_d>0$, probability $p_t = {3\log{t}}/
    {t}$, regularizer $\lambda$ \\
    \STATE \textbf{Initialize server:} 
    complete graph $G_0=(\mathcal{U}, E_0)$ over agents. $T_{u,0}^{ser}=0$, $\bm{b}_{u,0}^{ser}=\bm{b}_{u,0}^{clu}=\bm{0}_{d\times 1}$, $\hat{\bm{\theta}}_{u,0}^{ser}=\hat{\bm{\theta}}_{u,0}^{clu}=\bm{0}_{d\times 1}$, $\bm{\Sigma}_{u,0}^{ser}=\bm{\Sigma}_{u,0}^{clu}=\bm{0}_{d \times d}$\\
    \STATE \textbf{Initialize agents}: $T_{u,0}^{loc}=0$, $\bm{b}_{u,0}=\bm{b}_{u,0}^{loc}=\bm{0}_{d\times 1}$, $\hat{\bm{\theta}}_{u,0}=\bm{0}_{d\times 1}$, $\bm{\Sigma}_{u,0}=\bm{\Sigma}_{u,0}^{loc}=\bm{0}_{d \times d}$\\
    \FOR{round $t=1,...,T$}
    \STATE User $u_{t}$ arrives to be served \label{line:main_arrive}
	\STATE ComInd $\gets$\texttt{LocalAgent}($t, u_t, p_t, \alpha_c$)\label{line:main_call_sub}
    \IF{ComInd$==$1}
    \STATE Agent $u_t$ sends $\bm{\Sigma}_{u_t,t}^{loc}$\label{line:main_send}, $T_{u_t,t}^{loc}$ and $\bm{b}_{u_t,t}^{loc}$ to server 
	\STATE Run \texttt{Server}($t, u_t, \alpha_d$)
	\STATE Server sends $\bm{\Sigma}^{clu}_{u_t,t+1}$ and $\hat{\bm{\theta}}^{clu}_{u_t,t+1}$ back to agent $u_t$\label{line:server_send}
 \STATE Agent $u_t$ update:  $\bm{\Sigma}^{loc}_{u_t,t+1}=\bm{0}_{d \times d}$, $\bm{b}^{loc}_{u_t,t+1}=\bm{0}_{d \times 1}$, $T_{u_t,t}^{loc}=0$, 
 $\bm{\Sigma}_{u_t,t+1}=\bm{\Sigma}^{clu}_{u_t,t+1}$, $\hat{\bm{\theta}}_{u_t,t+1}=\hat{\bm{\theta}}^{clu}_{u_t,t+1}$
    \ENDIF
    \ENDFOR 
    \end{algorithmic}
    \end{minipage}
}
    \end{algorithm}

\begin{algorithm}[htb] 
\resizebox{0.97\columnwidth}{!}{
\begin{minipage}{\columnwidth}
    \caption{\texttt{LocalAgent}($t, u_t, p_t, \alpha_c$)} 
    \label{alg:local} 
    \begin{algorithmic}[1] 
	\STATE  Receives context $\bm{x}_{t,i}, \forall i\in \mathcal{I}_t$\label{line:local_context}
	\STATE 
    $U_{t}(i) \gets\min\left\{\hat{\bm{\theta}}_{{u_t},t}^{T}\bm{x}_{t,i}+\beta\left \| \bm{x}_{t,i} \right \|_{\bm{\Sigma}_{u_{t},t}^{-1}}, 1\right\}, \forall i\in \mathcal{I}_t$\label{line:local_ucb}
    \STATE $\bm{a}_{t} \gets oracle((U_{t}(i))_{i \in \mathcal{I}_t})$\label{line:local_rank}
    \STATE Play $\bm{a}_{t}$ and observe $O_{t}, w_{t}(a_{k,t})$ for $k \le O_{t}$, and receive reward $f(\bm{a}_{t},\bm{w}_{t})$
    \STATE $T_{u_t,t}^{loc}=T_{u_t,t-1}^{loc}+O_{t}$ \label{line:local_update_start}
    \STATE 
    $\bm{\Sigma}_{u_{t},t}^{loc} = \bm{\Sigma}_{u_{t},t-1}^{loc}+\sum_{k=1}^{O_{t}}\bm{x}_{t,a_{k,t}}\bm{x}_{t,a_{k,t}}^{\top}$, $\bm{b}_{u_{t},t}^{loc} = \bm{b}_{u_{t},t-1}^{loc}+\sum_{k=1}^{O_{t}}w_{t}(a_{k,t})\bm{x}_{t,a_{k,t}}$\label{line:local_update_end}
    \FOR{$u \ne u_t$}
    \STATE $T_{u,t}^{loc}=T_{u,t-1}^{loc}$ 
    \STATE $\bm{\Sigma}^{loc}_{u,t}=\bm{\Sigma}^{loc}_{u,t-1}$, $\bm{b}^{loc}_{u,t}=\bm{b}^{loc}_{u,t-1}$
    \STATE $\bm{\Sigma}_{u,t+1}=\bm{\Sigma}_{u,t}$, $\bm{b}_{u,t+1}=\bm{b}_{u,t}$, $\hat{\bm{\theta}}_{u,t+1}=\hat{\bm{\theta}}_{u,t}$
    \ENDFOR
    \STATE Generate a uniform random number $p\in [0,1]$
    \IF{det($\bm{\Sigma}_{u_{t},t}+\bm{\Sigma}_{u_t,t}^{loc}$) $>(1+\alpha_c)$det($\bm{\Sigma}_{u_{t},t}$) \textbf{or} $p<p_t$} \label{line:local_com_det}
    \STATE {\bf Return:} 1
    \ELSE
    \STATE $\bm{\Sigma}_{u_t,t+1}=\bm{\Sigma}_{u_t,t}$, $\bm{b}_{u_t,t+1}=\bm{b}_{u_t,t}$, $\hat{\bm{\theta}}_{u_t,t+1}=\hat{\bm{\theta}}_{u_t,t}$
    \STATE $G_{t+1}=G_t$
    \STATE {\bf Return:} 0
    \ENDIF
    \end{algorithmic}
    \end{minipage}
}
    \end{algorithm}

\begin{algorithm}[htb] 
\resizebox{0.97\columnwidth}{!}{
\begin{minipage}{\columnwidth}
    \caption{\texttt{Server}($t,u_t,\alpha_d$)} 
    \label{alg:server} 
    \begin{algorithmic}[1]
    \STATE From graph $G_{t}=(\mathcal{U}, E_{t})$,  server identifies the connected component $\mathcal{V}_t$ that user $u_t$ belongs to
    \STATE $\bm{\Sigma}^{clu}_{u_t,t+1}=\lambda I+\sum_{u\in \mathcal{V}_t}\bm{\Sigma}_{u,t}^{ser}$, $\ \bm{b}^{clu}_{u_t,t+1}=\sum_{u\in \mathcal{V}_t}\bm{b}_{u,t}^{ser}$, $\ \hat{\bm{\theta}}^{clu}_{u_t,t+1}=(\bm{\Sigma}_{u_t,t+1}^{clu})^{-1}\bm{b}^{clu}_{u_t,t+1}$\label{line:global_update_start}
    \STATE $T_{u_t,t}^{ser}=T_{u_t,t-1}^{ser}+T_{u_t,t}^{loc}$
    \STATE $\bm{\Sigma}^{ser}_{u_t,t}=\bm{\Sigma}^{ser}_{u_t,t-1}+\bm{\Sigma}^{loc}_{u_t,t}$, $\bm{b}^{ser}_{u_t,t}=\bm{b}^{ser}_{u_t,t-1}+\bm{b}^{loc}_{u_t,t}$, $\hat{\bm{\theta}}^{ser}_{u_t,t}=(\lambda I+\bm{\Sigma}^{ser}_{u_t,t})^{-1}\bm{b}^{ser}_{u_t,t}$\label{line:global_update_end}

     \STATE Reset deletion set $\tilde{E}=\emptyset$
    \FOR{$u \ne u_t$ and $(u,u_t) \in E_t$}
    \STATE $T_{u,t}^{ser}=T_{u,t-1}^{ser}$
    \STATE $\bm{\Sigma}^{ser}_{u,t}=\bm{\Sigma}^{ser}_{u,t-1}$, $\bm{b}^{ser}_{u,t}=\bm{b}^{ser}_{u,t-1}$, $\hat{\bm{\theta}}^{ser}_{u,t}=(\lambda I+\bm{\Sigma}^{ser}_{u,t})^{-1}\bm{b}^{ser}_{u,t}$
    \IF{$\left \|\hat{\bm{\theta}}_{u_t,t}^{ser}-\hat{\bm{\theta}}_{u,t}^{ser} \right\|>\alpha_d(\sqrt{\frac{1+\ln(1+T_{u_t,t}^{ser})}{1+T_{u_t,t}^{ser}}}+\sqrt{\frac{1+\ln(1+T_{u,t}^{ser})}{1+T_{u,t}^{ser}}})$}\label{line:server_check} 
    \STATE  Server puts edge $(u_t,u)$ into deletion set $\tilde{E}$ \label{line:server_delete}
    \ENDIF
	\ENDFOR
	\STATE Update $E_{t+1}=E_t \backslash \tilde{E}$ and obtain new graph $G_{t+1}=(\mathcal{U}, E_{t+1})$\label{line:graph_update} 
    \end{algorithmic}
    \end{minipage}
}
    \end{algorithm}

\section{The Proposed Algorithm}
This section presents our algorithm, named ``Federated Contextual Combinatorial 
Cascading Upper Confidence Bound Algorithm for Heterogeneous Users'' (FedC$^3$UCB-H), 
which is outlined in \cref{alg:main}. 
High levelly, the main task of FedC$^3$UCB-H is to identify which agents 
can collaborate based on user heterogeneity. 
Agents who can collaborate aggregate their data asynchronously to address the agent-side cascading bandit problems. FedC$^3$UCB-H consists of three main components as follows. For simplicity, we summarize the related notations in \cref{tab:notation}.

\noindent\textbf{Agent-side combinatorial action selection.} This component aims to select combinatorial action based on data received from the server. In each round $t \in [T]$, a random user $u_t$ arrives, and the corresponding agent $u_t$ will interact with this user (line \ref{line:main_arrive} in \cref{alg:main}).  
The agent $u_t$ then calls the subroutine \cref{alg:local} (line \ref{line:main_call_sub} in \cref{alg:main}). During this process, $u_t$ receives the feature vector $\bx_{t,i}$ for each item $i \in \cI_t$ (line~\ref{line:local_context} in \cref{alg:local}). Utilizing the estimated preference vector $\hat{\btheta}_{u_t, t}$ and the gram matrix $\bSigma_{u_t,t}$, the agent calculates an optimistic UCB value $U_t^{(u_t)}(i)$ for the true weight $\bar{w}_t(i)$ of each item  $i$ (line \ref{line:local_ucb} in \cref{alg:local}). This UCB value balances exploration and exploitation, with $\beta\norm{\bx_{t,i}}_{\bSigma_{u_t,t}^{-1}}$ accounting for the uncertainty of the estimated weight $\inner{\hat{\btheta}_{u_t, t}, \bx_{t,i}}$ and encouraging more exploration when it is large. 
It is crucial to note that $\hat{\btheta}_{u_t, t}$ and $\bSigma_{u_t,t}$ not only consist of data from user $u_t$ but also from users within the same estimated cluster $\mathcal{V}$ that includes $u_t$, which is identified by the heterogeneity testing at the server to be introduced later. 
Due to the data inconsistency caused by the asynchronous communication, the choice of $\beta$ is carefully designed and theoretically supported by \cref{lem:concentration}.  

After computing UCB values for each arm, the agent employs the computation oracle to generate a ranked list ${\bm a}_t \in \cA_t$ (line \ref{line:local_rank} in \cref{alg:local}).  
The user then scans through this list ${\bm a}_t$, receives the reward 
$f({\bm a}_t,\bw_t)$ (\cref{eq:reward_func}) according to the random weight $\bw_t$, and observes the partial feedback $O_t$ as defined in \cref{eq:feedback_model}. After receiving the feedback, agent $u_t$ updates its local data and stores them in the local buffer (lines \ref{line:local_update_start}-\ref{line:local_update_end} in \cref{alg:local}). Finally, based on certain conditions, $u_t$ evaluates whether it should communicate with the server (line \ref{line:local_com_det} in \cref{alg:local}), which will be covered later in our asynchronous communication protocol.

\noindent\textbf{Server-side user heterogeneity testing.} Different from \citet{2019Distributed,he2022simple} where users are assumed to be homogeneous, the server in our setting must address the user heterogeneity.  
Specifically, the server's task is to partition users into distinct clusters. Users from different clusters are viewed as heterogeneous and should not collaborate with each other. To address this task, 
the server maintains an undirected graph $G_t=(\mathcal{U}, E_t)$ over agents. This graph connects agents via edges if they are estimated to belong to the same cluster.
Initially, $G_0$ is initialized as a complete graph, and it will be updated adaptively based on the uploaded information after each communication. 
At round $t$, if the communication condition in line \ref{line:local_com_det} of \cref{alg:local} is satisfied with communication indicator ComInd $=1$, $u_t$ uploads its local data (line~\ref{line:main_send} in \cref{alg:main}). The server then updates the agent $u_t$'s information (lines~\ref{line:global_update_start}-\ref{line:global_update_end} in \cref{alg:server}). 
Next, the server verifies the estimated heterogeneity between user $u_t$ and other users based on the updated estimation.  
In particular, for each user $u$ connected to $u_t$ via edge $(u_t,u) \in E_{t-1}$, if the gap between her estimated preference vectors $\hat{\btheta}^{ser}_{u,t}$ and $\hat{\btheta}^{ser}_{u_t,t}$ exceeds a certain threshold (line \ref{line:server_check} in \cref{alg:server}), 
the server removes edge $(u_t, u)$ to separate them (line \ref{line:server_delete} in \cref{alg:server}).  This threshold is carefully set to ensure that, with high probability, edges between heterogeneous users (those not in the same \gtclusters{}) are deleted after a short span of time steps. Conversely, edges connecting users within the same \gtclusters{} are retained. After user clusters are correctly identified, our algorithm will leverage collaborative information between homogeneous users while avoiding misuse of information from heterogeneous users to reduce the overall uncertainty.
In line \ref{line:graph_update} of \cref{alg:server}, the server uses graph $G_t$ to identify the connected component $\cV_t$ in which $u_t$ resides, and sends back to $u_t$ the updated collaborative information of $\cV_t$ for future decisions.

\noindent\textbf{Asynchronous communication protocol.}  This protocol addresses the key question of when $u_t$ should communicate with the server to share her information and to keep her updated with the newest information.  Striking the right balance is the main challenge because excessive communication incurs high communication costs, while insufficient communication results in inaccurate estimation for the local agents and prevents the server from updating graph $G_t$ effectively.

The communication protocol involves a two-step verification process to decide whether the agent needs to activate a communication. Inspired by \citet{he2022simple}, the first step is to verify a matrix determinant-based condition (line \ref{line:local_com_det} in \cref{alg:local}). This condition gauges the gap between information accumulated locally and collaborative information shared during the previous communication. If satisfied, it signifies that enough local data have been collected to significantly reduce global model uncertainty. Thus, agent $u_t$ triggers a communication with the server, leading to global model updates (line \ref{line:global_update_start} in \cref{alg:server}). Subsequently, the server provides the latest collaborative information to agent $u_t$ (line \ref{line:server_send} in \cref{alg:main}). Notably, this communication indirectly aids the update of graph $G_t$ (line \ref{line:graph_update} in \cref{alg:server}).

Our communication protocol involves a two-step verification process to decide whether a communication needs to be activated. Inspired by \citet{he2022simple}, the first step is to verify a matrix determinant-based condition (line \ref{line:local_com_det} in \cref{alg:local}). This condition measures the gap between the information accumulated locally and the collaborative information shared during the previous communication.  
If satisfied, it means that enough local data have been collected to significantly reduce global model uncertainty. 
Thus, agent $u_t$ will trigger a communication with the server, leading to the update of the global model at the server (line \ref{line:global_update_start} in \cref{alg:server}). 
Subsequently, the server returns the lastest collaborative information to agent $u_t$ (line \ref{line:server_send} in \cref{alg:main}). Note that as a by-product, this communication will also help to update the graph $G_t$ (line \ref{line:graph_update} in \cref{alg:server}).
However, this by-product communication isn't enough to ensure that the server's heterogeneity testing of all agents is completed in sublinear time, 
which motivates our second condition to increase the number of communications. Specifically, we introduce a certain probability $p_t$ that user $u_t$ 
activates an auxiliary communication at round $t$, as a complement to the matrix determinant-based condition. We refer to this protocol as \textit{$p_t$-auxiliary communication protocol}. The key challenge is to determine suitable $p_t$, and through careful analysis, we set $p_t={3\log t}/{t}$. A deterministic version of this protocol, the force-communication protocol, is also offered for uncertainty-averse applications. Note that we also provide a deterministic version of the $\epsilon$-auxiliary condition, i.e., the force-communication protocol, for uncertainty-averse applications. However, our theory and experiments show that the force-communication may yield worse performance both in regret and communication, whose details and analysis are postponed to the Appendix. In cases where neither the matrix determinant-based condition nor the $p_t$-auxiliary condition is met, communication between the agent and the server is not activated, and local data remain local. For inactive agents, their data also remains unchanged.

\section{Theoretical Analysis}
\begin{theorem}\label{thm:main}
Under Assumptions 1-3 stated in the section of Problem Setup, if we set $\beta$ as in \cref{lem:concentration} and $\tilde{\lambda}_x\triangleq\int_{0}^{\lambda_x} (1-e^{-{(\lambda_x-x)^2}/{2\sigma^2}})^{I} dx$, then the expected
cumulative regret of \cref{alg:main} is upper bounded by
\begin{align}
 \text{Reg}(T) &\le O\big(d\sqrt{(J+|\cU|\alpha_c)(1+|\cU|^2\alpha_c)KT}\log T\notag\\
 &+ |\cU|({d}\gamma^{-2}\tilde{\lambda}_x^{-1}+\tilde{\lambda}^{-2}_x)\log{T}\big).
\end{align} 
The communication cost is bounded by 
$
    \text{Com}(T) \le O\left(d(\left | \mathcal{U} \right |+1/\alpha_c)\log T+ \log^2 T\right)
$
\end{theorem}

\noindent \textbf{Discussion and Comparison.}
Looking at the above theorem, by setting $\alpha_c=1/\left | \mathcal{U} \right |^2$, we can bound the regret by $\tilde{O}(d\sqrt{JKT})$, while maintaining a near-constant $\tilde{O}(d|\cU|^2)$ communication cost, which is exactly the result presented in row 8 of \cref{tab:fed_res2}. Compared with the state-of-the-art CLUB-cascade algorithm \cite{li2018online} which incurs a $O(T)$ communication cost (row 4 of \cref{tab:fed_res2}), our algorithm significantly reduces the communication cost to logarithmic terms regarding $T$, while obtaining the same regret bounds.  In the scenario where users are homogeneous ($J=1$) and only a single arm is selected per round ($K=1$), our regret matches \citet{he2022simple}, incurring only an additional $O(\log^2T)$ communication cost due to the $p_t$-auxiliary communication. As for the lower bound, \citet{liu2022federated} establishes a lower bound of $\Omega(\sqrt{dJT})$ with unlimited communication, where our result is near-optimal and matches this lower bound up to a factor $O(\sqrt{dK})$. For the communication cost, \citet{he2022simple} shows $\Omega(|\cU|)$ communication is needed to achieve near-optimal regret, and removing the $O(d|\cU|)$ communication gap is still a challenging open problem.

\begin{proof}[Proof for \cref{thm:main}] Due to space limit, we give a sketched proof here. Our regret analysis mainly contains two parts. The first part upper bounds the number of exploration rounds required to achieve accurate user partitioning at the server. In particular, leveraging assumptions of item regularity and user arrivals, we prove that the user groups would be correctly partitioned if all user data are synchronized and uploaded to the server exactly at (or after) $T_0=O(|\cU|({d}\gamma^{-2}\tilde{\lambda}_x^{-1}+\tilde{\lambda}^{-2}_x)\log{T})$ timesteps. However, since $\gamma$ and $\tilde{\lambda}_x$ are unknown (which makes $T_0$ unknown) to the agents or the server, we devise a $p_t$-auxiliary communication that is agnostic to $T_0$ and ensures that after at most $T_1=O(|\cU|\log T + \sqrt{T})$ extra rounds, the server obtains the correct user partition. For the second part of our analysis after $T_0+T_1$, agents sharing the same $\btheta$ begin to collaborate with each other under server coordination and we prove the following concentration bound:
\begin{lemma}\label{lem:concentration}
    Let $\beta=\sqrt{\lambda}+(\sqrt{1+\left | \mathcal{U} \right |\alpha_c}+\left | \mathcal{U} \right |\sqrt{2\alpha_c})\sqrt{d\log(1+\frac{KT}{\alpha_c \lambda d})+2\log({1}/{\delta})+4\log(T\left | \mathcal{U} \right |)}$, for $t > T_0+T_1$ and every $u \in \mathcal{U}$ in true user cluster $j(u)$, it holds with probability at least $1-\delta$ that $\norm{\hat{\bm{\theta}}_{u,t}-\bm{\theta}^{j(u)}}_{\bm{\Sigma}_{u,t}}\leq \beta$.
\end{lemma}
The key challenge of proving \cref{lem:concentration} is to handle the gap between the delayed $\bm{\Sigma}_{u,t}$ and up-to-date information had agents been synchronized. We address the challenge by bounding this gap by a factor of $O(J+|\cU|\alpha_c)$ thanks to the matrix determinant condition in line \ref{line:local_com_det} of \cref{alg:local}. Based on this concentration bound, we conduct a contextual combinatorial cascading bandit regret analysis to bound the final regret. For the communication cost, the matrix-determinant condition and the $p_t$-auxiliary communication contribute $O(d(|\cU|+1/\alpha_c)\log T)$ and $O(\log^2 T)$, respectively. For detailed proofs, please refer to our Appendix.
\end{proof}

\begin{figure}[t]
\begin{subfigure}{.23\textwidth}
  \centering
  % include first image
  \includegraphics[width=\linewidth]{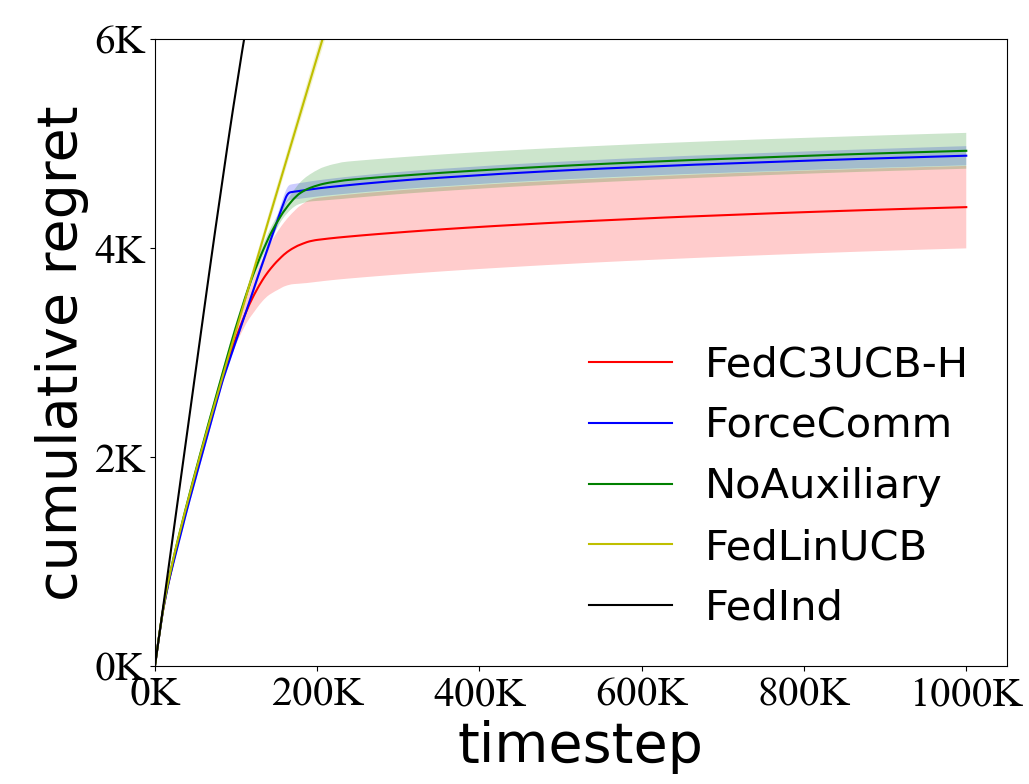}  
  \caption{Synthetic, $J=5$}
  \label{fig:J=5}
\end{subfigure}
\begin{subfigure}{.23\textwidth}
  \centering
  % include second image
  \includegraphics[width=\linewidth]{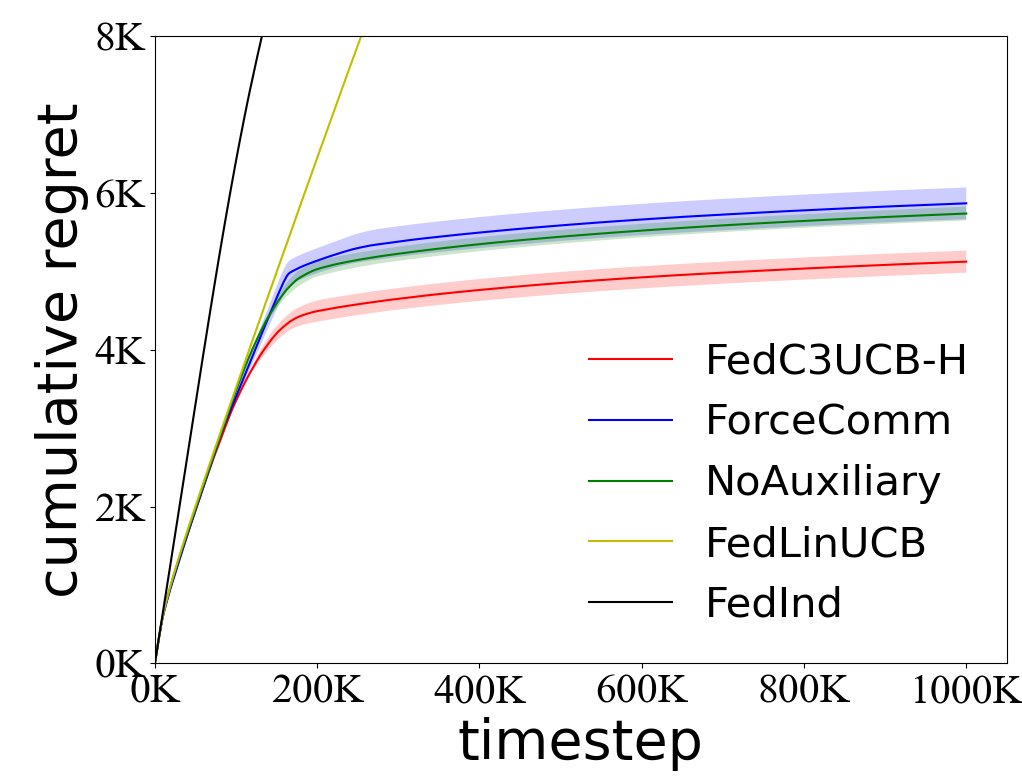}  
  \caption{Synthetic, $J=10$}
  \label{fig:J=10}
\end{subfigure}

% \newline

\begin{subfigure}{.23\textwidth}
  \centering
  % include third image
  \includegraphics[width=\linewidth]{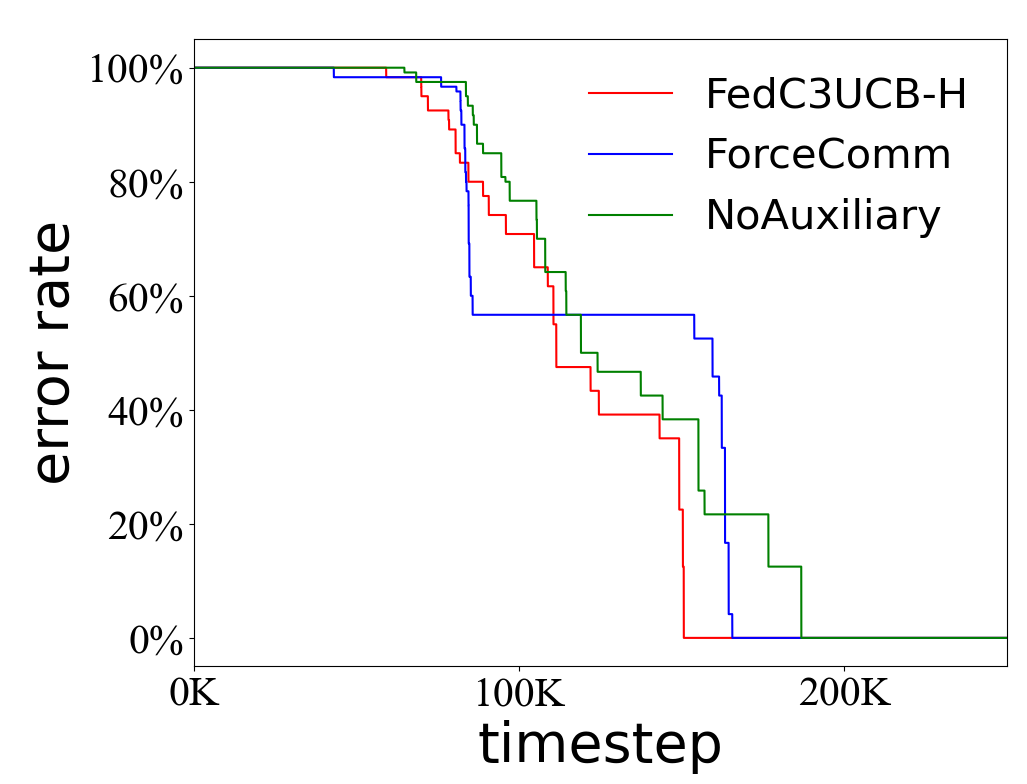}  
  \caption{Synthetic, $J=10$}
  \label{fig:cluster}
\end{subfigure}
\begin{subfigure}{.23\textwidth}
  \centering
  % include fourth image
  \includegraphics[width=\linewidth]{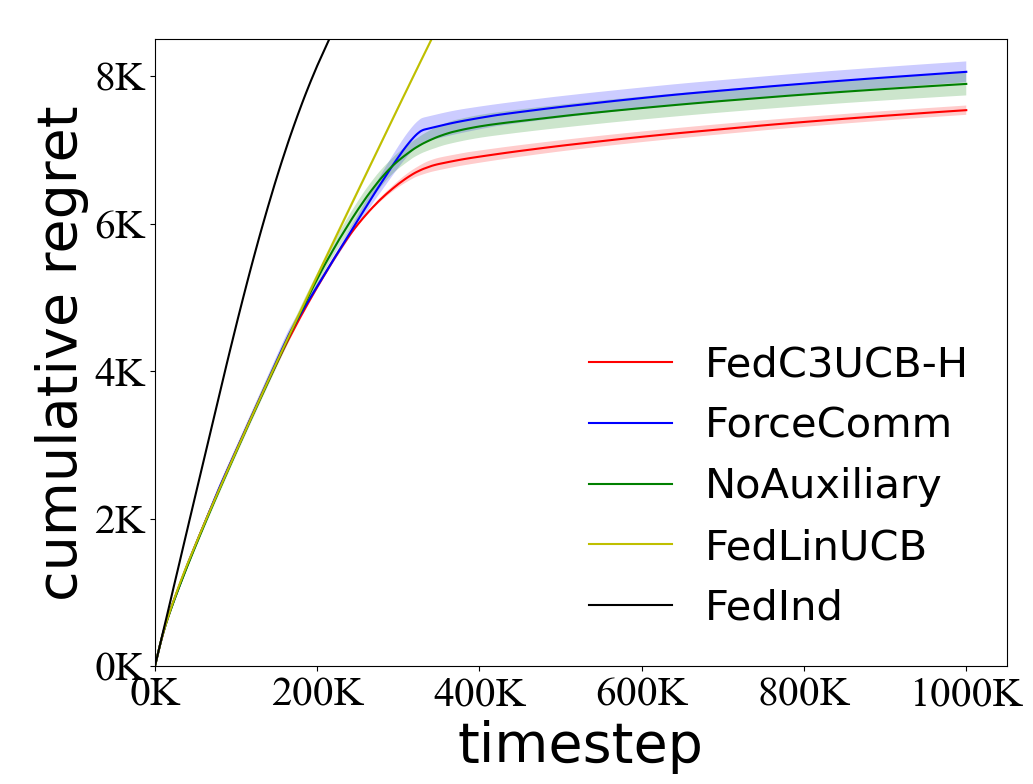}  
  \caption{Yelp, $J=10$}
  \label{fig:Yelp}
\end{subfigure}
\caption{Cumulative regret and error rates of user clustering}
\label{fig:exp}

\end{figure}

\section{Empirical Evaluation}
In this section, we compare FedC$^3$UCB-H against four baseline algorithms: (1) FedLinUCB (\citet{he2022simple}), adapted to accommodate cascading arms while maintaining a single estimated vector for all agents, (2)  FedInd, where agents operate independent cascading bandit algorithms locally without any communication, (3) NoAuxiliary, a variant of FedC$^3$UCB-H that omits the $p_t$-auxiliary communication condition, and (4) ForceComm, which replaces the $p_t$-auxiliary communication with a deterministic force-communication protocol. In the ForceComm protocol, each agent communicates with the server upon its arrival at time steps $\{1, 2, 4, ..., 2^n\}$. The baselines (1) and (2) aim to evaluate the significance of user cluster identification, while the baselines (3) and (4) show the importance of our $p_t$-auxiliary communication.
We report the averaged regret, communication cost, and error rates of user clustering (the number of incorrect users divided by $|\cU|$; users in a correctly partitioned cluster are considered correct, while others are incorrect) over three independent runs with standard deviations.

\noindent\textbf{Synthetic Experiments.}
We conduct experiments in a synthetic environment with $|\cU|=40$ users, $J=5$ or $J=10$ clusters with orthogonal $\btheta$, and $T=1000,000$ rounds. The preference and feature vectors are of dimension $d=20$, with each entry drawn from a standard Gaussian distribution, and are normalized to vectors with $\norm{.}_2=1$~\cite{li2019improved}. At round $t$, a user $u_t$ randomly chosen from $\cU$ and and $I_t=200$ items are generated, where each item is associated with a random $\bx_{t,i} \in \R^{d}$ generated as above. The agent needs to recommend $K=4$ items as cascading arms $\ba_t$ to the user, leading to a random reward $f(\ba_t,\bw_t)$ as defined in \cref{eq:reward_func}. The agents then observed feedback according to \cref{eq:feedback_model}. 
Results for the $J=5$ scenario are shown in \cref{fig:J=5}, while \cref{fig:J=10} depicts the results for $J=10$. In both cases, FedC$^3$UCB-H outperformes the baseline algorithms. For the $J=10$ setting, FedC$^3$UCB-H achieves an 82\% reduction in regret compared to FedLinUCB  and a 64\% reduction compared to FedInd. This shows the effectiveness of our user clustering procedure. Compared with NoAuxiliary, FedC$^3$UCB-H achieves an 10\% reduction in regret, indicating the importance of our $p_t$-auxiliary communication. This is confirmed by \cref{fig:cluster}, where NoAuxiliary exhibited slower convergence in correctly identifying user clusters.
When compared to ForceComm, FedC$^3$UCB-H achieved 12\% less regret and 5\% less communication. Although ForceComm initially identifies user clusters more rapidly (\cref{fig:cluster}), it incurs unnecessary communication after users are correctly partitioned, causing it to perform worse over time. FedC$^3$UCB-H, instead, strikes a better balance. The $J=10$ scenario produces consistent results, reaffirming the observations in the $J=5$ scenario.

\begin{table}[t]
\resizebox{1\columnwidth}{!}{
\begin{minipage}{\columnwidth}
\center
\begin{tabular}{clclclcl}
\hline
\multicolumn{2}{c}{} & \multicolumn{4}{c}{Synthetic} & \multicolumn{2}{c}{Yelp}              \\
\multicolumn{2}{c}{}     & \multicolumn{2}{c}{$J=5$} & \multicolumn{2}{c}{$J=10$} &  \multicolumn{2}{c}{$J=10$}\\ \hline
\multicolumn{2}{c}{FedC$^3$UCB-H} & \multicolumn{2}{c}{818.3}        & \multicolumn{2}{c}{1070.7}       & \multicolumn{2}{c}{951.3}     \\ \hline
\multicolumn{2}{c}{ForceComm}              & \multicolumn{2}{c}{909.3}        & \multicolumn{2}{c}{1125.3}     & \multicolumn{2}{c}{1034.0}      \\ \hline
\multicolumn{2}{c}{NoAuxiliary}              & \multicolumn{2}{c}{666.7}        & \multicolumn{2}{c}{899.3}    & \multicolumn{2}{c}{789.3}       \\ \hline
\multicolumn{2}{c}{FedLinUCB}              & \multicolumn{2}{c}{479.7}        & \multicolumn{2}{c}{482.3}        & \multicolumn{2}{c}{486.0}   \\ \hline
\end{tabular}
\caption{Communication cost for different datasets.}\label{tab:com}
\end{minipage}
}

\end{table}

\noindent\textbf{Experiments on Real Dataset.}
We conduct experiments on the Yelp dataset, which contains $4.7$ million ratings of $1.57\times 10^5$ restaurants from $1.18$ million users.\footnote{http://www.yelp.com/dataset challenge} Since Yelp dataset is very sparse, we extract $1000$ items with the most ratings and $1000$ users who rate most. This subset forms a new rating matrix $\bM^{1000\times 1000}$. We use matrix $\bM$ to generate feature vectors and preference vectors of $d=20$ dimension for all items and users by singular-value decomposition (SVD) \cite{li2018online,li2019improved,liu2022federated}, respectively. Then we randomly sample $|\cU|=40$ users and use k-means clustering \cite{ahmed2020k} to generate $J=10$ user clusters. Each cluster's center vector represents the true preference vector for users within that cluster. The experiment's remaining parameters match those of the synthetic experiments. \cref{fig:Yelp} shows the regret comparison and \cref{tab:com} shows the communication cost. The results are consistent with the synthetic experiments.  FedC$^3$UCB-H achievs a 67\% reduction in regret compared to FedLinUCB, 43\% compared to FedInd, 5\% compared to NoAuxiliary, and 6\% compared to ForceComm. As for the communication cost, FedC$^3$UCB-H has higher communication costs than NoAuxiliary and FedLinUCB due to auxiliary communication, but manages to reduce communication by 8\% compared to ForceComm.

\section{Future Directions}
It would be interesting to eliminate the $O(\sqrt{dK})$ gap in the regret and the $O(d|\cU|)$ gap in the communication. It would also be valuable to generalize the linear structure by 
considering non-linear structures or model mis-specifications.

\section*{Acknowledgments}
We thank the anonymous reviewers for their helpful comments. The work of John C.S. Lui was supported in part by the RGC GRF 14202923.

\bibliography{arixiv_v1}
\newpage
\appendix
\onecolumn
\section*{Appendix}
\section{Regret Analysis for the Regret and Communication Cost}
\subsection{Regret Analysis}
Our regret analysis mainly consists of two parts. The first part upper bounds the number of exploration rounds needed so that with high probability, the users are partitioned correctly at the server. This guarantees that the information will not be shared incorrectly among heterogeneous users with different preference vectors. The second part analyses the regret for asynchronous combinatorial cascading bandits after users are partitioned correctly. To conduct our analysis, we first introduce some definitions as follows.
We assume that there is an omnipotent analyst who can see the whole progress of the running algorithm, and thus we define the following auxiliary matrices and vectors that contain all the information up to timestep $t$ with respect to cluster $\mathcal{V}_j$:
\begin{equation}\label{eq.1a}
\bm{\Sigma}_{\mathcal{V}_j,t}^{all}=\lambda I+\sum_{i=1}^t \mathbb{I}\{u_i \in \mathcal{V}_j\}\sum_{k=1}^{O_i} \bm{x}_{i, a_{k, i}} \bm{x}_{i, a_{k, i}}^T,
\end{equation}
\begin{equation}\label{eq.1b}
\bm{b}_{\mathcal{V}_j,t}^{all}=\sum_{i=1}^t\mathbb{I}\{u_i \in \mathcal{V}_j\} \sum_{k=1}^{O_i} w_i\left(a_{k, i}\right) \bm{x}_{i, a_{k, i}},
\end{equation}
\begin{equation}\label{eq.1c}
\bm{g}_{\mathcal{V}_j,t}^{all}=\sum_{i=1}^t\mathbb{I}\{u_i \in \mathcal{V}_j\} \sum_{k=1}^{O_i} \eta_{i,a_{k,i}}^{(u_i)}\bm{x}_{i, a_{k, i}},
\end{equation}
\par
where $O_t$ is the cascading feedback defined in the section of Problem Setup, and $\eta_{t,i}^{(u_t)}=w_t(i)-\left \langle   \bm{\theta}_{u_t}, \bm{x}_{t,i} \right \rangle$. Because our setting for $w_t(i)$ for each item $i$ is sample from Bernoulli distribution, this means that $\eta_{t,i}^{(u_t)}$ is a $R$-sub-Gaussian noise (i.e., $R=1/2$) for each item $i$ and the current activated agent $u_t$, i.e.for all $b \in \mathbb{R}$ it holds that $\mathbb{E}\left [ \text{exp}(b\eta_{t,i}^{(u_t)})|\mathcal{H}_t \right ] \leq \text{exp}(b^2 R^2/2)$. For simplicity, because only one agent can be activated in each timestep, the agent index is determined by timestep, we omit the superscript $u_t$ when context is clear.\par
In this case, the omnipotent estimate for cluster $\mathcal{V}_j$ is:
\begin{equation}\label{eq.1d}
\hat{\bm{\theta}}_{\mathcal{V}_j,t}^{all}=\left(\bm{\Sigma}_{\mathcal{V}_j,t}^{all}\right)^{-1} \bm{b}_{\mathcal{V}_j,t}^{all}
\end{equation}
Let $N_m(t)$ represent the last timestep when agent $m$ communicated with the server at the end of timestep $t$. Under this definition, for all timestep $t$, the data uploaded and accumulated at the server for agent $u$ \textit{at the end of timestep $t$} are:
\begin{equation}\label{eq.2a}
    \bm{\Sigma}_{u, t}^{ser}=\sum_{i=1}^{N_u(t)}\mathbb{I}\{u_i=u\} \sum_{k=1}^{O_i}\bm{x}_{i, a_{k, i}} \bm{x}_{i, a_{k, i}}^T,
\end{equation}
\begin{equation}\label{eq.2b}
    \bm{b}_{u,t}^{ser}=\sum_{i=1}^{N_u(t)}\mathbb{I}\{u_i=u\} \sum_{k=1}^{O_i}w_i\left(a_{k, i}\right) \bm{x}_{i, a_{k, i}},
\end{equation}
\begin{equation}\label{eq.2c}
    \bm{g}_{u, t}^{ser}=\sum_{i=1}^{N_u(t)}\mathbb{I}\{u_i=u\} \sum_{k=1}^{O_i} \eta_{i,a_{k,i}}^{(u_i)} \bm{x}_{i, a_{k, i}}, 
\end{equation}
And the local data that agent $u$ has not been uploaded to the server are:
\begin{equation}\label{eq.3a}
    \bm{\Sigma}_{u, t}^{loc}=\sum_{i=N_u(t)+1}^t \mathbb{I}\{u_i=u\}\sum_{k=1}^{O_i} \bm{x}_{i, a_{k, i}} \bm{x}_{i, a_{k, i}}^T,
\end{equation}
\begin{equation}\label{eq.3b}
 \bm{b}_{u,t}^{loc}=\sum_{i=N_u(t)+1}^t \mathbb{I}\{u_i=u\}\sum_{k=1}^{O_i} w_i\left(a_{k, i}\right) \bm{x}_{i, a_{k, i}}, 
\end{equation}
\begin{equation}\label{eq.3c}
 \bm{g}_{u, t}^{loc}=\sum_{i=N_u(t)+1}^t \mathbb{I}\{u_i=u\}\sum_{k=1}^{O_i} \eta_{i,a_{k,i}}^{(u_i)} \bm{x}_{i, a_{k, i}},
\end{equation}
The real number of playing base arms of agent $u$ until timestep $t$ is (by ``real", we mean the up-to-date data of agent $u$):
\begin{equation}\label{eq.4}
T^{real}_{u,t}=T^{loc}_{u,t}+T_{u,t}^{ser}
\end{equation}
The real data of agent $u$ until timestep $t$ are:
\begin{equation}\label{eq.5}
\bm{\Sigma}^{real}_{u,t}=\bm{\Sigma}_{u, t}^{ser}+\bm{\Sigma}_{u, t}^{loc}, \quad \bm{b}_{u,t}^{real}=\bm{b}_{u,t}^{ser}+\bm{b}_{u,t}^{loc}, \quad
\bm{g}_{u, t}^{real}=\bm{g}_{u, t}^{ser}+\bm{g}_{u, t}^{loc}
\end{equation}
Obviously, for each cluster $\mathcal{V}_j$ the following equations hold:  
\begin{equation}\label{eq.6}
\bm{\Sigma}^{all}_{\mathcal{V}_j,t}=\lambda I+\sum_{u' \in \mathcal{V}_j}\bm{\Sigma}^{real}_{u',t},\quad \bm{b}_{\mathcal{V}_j,t}^{all}=\sum_{u' \in \mathcal{V}_j}\bm{b}_{u',t}^{real},\quad \bm{g}_{\mathcal{V}_j,t}^{all}=\sum_{u' \in \mathcal{V}_j}\bm{g}_{u',t}^{real}
\end{equation}
We define the real estimate as:
\begin{equation}\label{eq.7}
\hat{\bm{\theta}}_{\mathcal{V}_j,t}^{real}=\left(\lambda I+\bm{\Sigma}_{\mathcal{V}_j,t}^{real}\right)^{-1} \bm{b}_{\mathcal{V}_j,t}^{real} 
\end{equation}
Finally, we define the total number of playing base arm with respect to the cluster $\mathcal{V}_j$:
\begin{equation}\label{eq.8}
T^{all}_{\mathcal{V}_j,t}=\sum_{u' \in \mathcal{V}_j}T^{real}_{u',t} 
\end{equation}
At $t=T$, we write $T^{all}_{\mathcal{V}_j,T}$ as $T_{\mathcal{V}_j}^{all}$ for simplicity, and then $T_{\mathcal{V}_j}^{all}=\sum_{t=1}^T \mathbb{I}\{u_t \in \mathcal{V}_j\}O_t$.\par
Note that in our setting, the real data in \cref{eq.4}, \cref{eq.5} and \cref{eq.7} is not accessible by the server or agents due to asynchronous communication and this definition are purely used to facilitate our analysis. And we remark that when if the agent $u$ communicates with the server at timestep $t'$, then at the end of this timestep, we have:
\begin{equation}\label{eq.9}
T_{u,t'}^{ser}=T^{real}_{u,t'},\quad \bm{\Sigma}^{real}_{u,t'}=\bm{\Sigma}_{u, t'}^{ser},\quad \bm{b}_{u,t'}^{real}=\bm{b}_{u,t'}^{ser}, \quad \bm{g}_{u, t'}^{real}=\bm{g}_{u, t'}^{ser}, \quad \hat{\bm{\theta}}_{u, t'}^{real}=\hat{\bm{\theta}}_{u, t'}^{ser} 
\end{equation}
\par
For the first user partition part, by the assumption of item regularity and the user arrival assumption, we prove that after $T_0=O(\left | \mathcal{U} \right |(\frac{d}{\gamma^2\tilde{\lambda}_x}+\frac{1}{\tilde{\lambda}^2_x})\log{T})$ timesteps (where $\tilde{\lambda}_x$ is given in \cref{apdx_lem:new_lambda_x}), the $l_2$ distance $\left \|\hat{\btheta}_{u,t}^{real}-\btheta^{j(u)}\right \| _2$ is less or equal to $\gamma/2$ for all users $u$, where $\gamma$ is the minimum heterogeneity gap defined in Assumption 1. If we consider the synchronous setting where users' data are synchronized at each step and sync time holds for each round $t$, $T_0$ is sufficient to guarantee that all users are partitioned correctly \cite{li2018online}. In our asynchronous setting, however, we need to guarantee that each agent communicates with the server $\textit{at least once}$ after $T_0$, so that sync time holds for every agent $u$ and the data accumulated before $T_0$ can successfully reach the server to obtain the correct user partition. By our $p_t$-auxiliary communication protocol, we can prove that after additional $T_1=O(\sqrt{T})$ rounds following $T_0$, each agent communicates with the server at least once with high probability. For the deterministic version of the $p_t$-auxiliary communication protocol, $T_1=O(\sqrt{T})$ is sufficient to obtain the correct partition at the server.\par
After $t\ge T_0+T_1$ when users are partitioned correctly, users within the same user cluster will collaborate and share information with each other. Thanks to the first-step verification in line \ref{line:local_com_det} in \cref{alg:local}, the gap of the covariance matrix of the virtual information $\bm{\Sigma}_{u_t,t}$ and the information used by the agent $\bm{\Sigma}_{u_t,t}^{loc}+\bm{\Sigma}_{u_t,t}$ is within a coefficient $1+\alpha_c$. This guarantees that the regret after $T_0+T_1$ rounds is upper bounded by a $2\sqrt{J+\left | \mathcal{U} \right |\alpha_c}$ factor compared to the regret of $J$ independent combinatorial cascading bandits happening.
Based on the definition shown above, we now give our detailed proof for \cref{thm:main}.
\begin{proof}[Proof of \cref{thm:main}]
First, we define the events:
\[
\varepsilon = \{\text{the clusters are correct for all}\ t>T_1+T_0 \}
\]
\[
\mathcal{F}_u(\delta)=\{ \left\|\hat{\bm{\theta}}_{u,t}-\bm{\theta}^{j(u)}\right\|_{\bm{\Sigma}_{u,t}}\leq \beta(\delta),\text{ when cluster is correct}\}
\]\par
We use $R^{\alpha}_{u_t}(t,\bm{a}^{(u_t)}_{t})$ to denote the per-round regret defined in \cref{apdx_eq:per-round_regret}, and it is easy to show that $R^{\alpha}_{u_t}(t,\bm{a}^{(u_t)}_{t})\leq 1$ under our setting.
The expected cumulative regret is 
\begin{align}
        Regret(T)
        &=\mathbb{E}[\sum_{t=1}^TR^{\alpha}_{u_t}(t,\bm{a}^{(u_t)}_{t})] \notag \\
        &=\mathbb{E}[\sum_{t=1}^{T_0}R^{\alpha}_{u_t}(t,\bm{a}^{(u_t)}_{t})]+\mathbb{E}[\sum_{t=T_0+1}^{T_0+T_1}R^{\alpha}_{u_t}(t,\bm{a}^{(u_t)}_{t})]+\mathbb{E}[\sum_{t=T_0+T_1+1}^{T}R^{\alpha}_{u_t}(t,\bm{a}^{(u_t)}_{t})]  \notag \\
        &\leq T_0+T_1+\mathbb{E}[\sum_{t=T_0+T_1+1}^{T}R^{\alpha}_{u_t}(t,\bm{a}^{(u_t)}_{t})] \label{eq.10a} \\
&=T_0+T_1+P(\varepsilon,\mathcal{F}_1(\delta),...,\mathcal{F}_{\left | \mathcal{U} \right |}(\delta))\mathbb{E}[\sum_{t=T_0+T_1+1}^{T}R^{\alpha}_{u_t}(t,\bm{a}^{(u_t)}_{t})|\varepsilon,\mathcal{F}_1(\delta),...,\mathcal{F}_{\left | \mathcal{U} \right |}(\delta)] \notag \\
&+P(\varepsilon^c\cup\bigcup_{u=1}^{\left | \mathcal{U} \right |} \mathcal{F}_u^c(\delta))\mathbb{E}[\sum_{t=T_0+T_1+1}^{T}R^{\alpha}_{u_t}(t,\bm{a}^{(u_t)}_{t})|\varepsilon^c\cup\bigcup_{u=1}^{\left | \mathcal{U} \right |} \mathcal{F}_u^c(\delta)] \label{eq.10b} \\
&\leq T_0+T_1+\mathbb{E}[\sum_{t=T_0+T_1+1}^{T}R^{\alpha}_{u_t}(t,\bm{a}^{(u_t)}_{t})|\varepsilon,\mathcal{F}_1(\delta),...,\mathcal{F}_{\left | \mathcal{U} \right |}(\delta)] \notag \\
&+(T-T_0-T_1)(P(\varepsilon^c)+\sum_{u=1}^{{\left | \mathcal{U} \right |}}P(\mathcal{F}_u^c(\delta))), \label{eq.10c}
\end{align}\par
where \cref{eq.10a} and \cref{eq.10c} is because $R^{\alpha}_{u_t}(t,\bm{a}^{(u_t)}_{t})\leq 1$ and \cref{eq.10b} is due to the property of conditional expectation. Note that we omit the sum of ${\left | \mathcal{U} \right |}$ agents, because at each timestep only one agent is activated, so the activated agent is determined by the timestep $t$.\par
Now we need to determine the explicit quantity of $T_0 + T_1$, which represents the timestep that the server partitions the cluster correctly with high probability.\\
We define the events:
\[
B_{1u}(\delta)=\{ \left\|\hat{\bm{\theta}}_{u,t}^{real}-\bm{\theta}^{j(u)}\right\|_{\bm{\Sigma}_{u,t}^{real}+\lambda I}\leq \beta'(T^{real}_{u,t},\delta),\ \forall t \}
\]
\[
B_{2u}(\delta)=\{\lambda_{min}(\bm{\Sigma}^{real}_{u,t})\geq T^{real}_{u,t}\tilde{\lambda}_x/8,\text{ for all } T^{real}_{u,t}\geq \frac{1024}{\tilde{\lambda}_x^2}\log{\frac{512d}{\tilde{\lambda}_x^2 d}} \}
\]
\[
B_{3u}(\delta)=\{\text{agent }u \text{ communicate with the sever at least once time in } [T_0,T_0+T_1] \text{ when }T^{activated}_{u,T_0 \sim T_0+T_1}\geq \sqrt{1/\delta} \}
\]
where $\beta'(T^{real}_{u,t},\delta)=R\sqrt{d\log(1+\frac{T^{real}_{u,t}}{\lambda d})+2\log\frac{1}{\delta}}+\sqrt{\lambda}S$ and $T^{activated}_{u,T_0 \sim T_0+T_1}$ is the number of agent $u$ be activated as timestep $t \in [T_0,T_0+T_1]$.\par
According to \citet{10.5555/2986459.2986717} and \cref{apdx_lem:new_lambda_x}, we know that $P(B_{1u}^c(\delta))\leq \delta$. According to \citet{li2018online}, we can know $P(B_{2u}^c(\delta))\leq \delta$. Then we take the event $\{B_{1u}(\delta/(4{\left | \mathcal{U} \right |})),B_{2u}(\delta/(4{\left | \mathcal{U} \right |}))\text{ hold for all agent }u \in \mathcal{U}\}$, this event happens with probability at least $1-3\delta/4$.\par
When the event $B_{1u}(\delta/(4{\left | \mathcal{U} \right |}))$ and $B_{2u}(\delta/(4{\left | \mathcal{U} \right |}))$ both hold for all agent
\[
\left\|\hat{\bm{\theta}}_{u,t}^{real}-\bm{\theta}^{j(u)}\right\|_2 \leq \frac{\left\|\hat{\bm{\theta}}_{u,t}^{real}-\bm{\theta}^*_{\mathcal{V}_j(u)}\right\|_{\bm{\Sigma}_{u,t}^{real}+\lambda I}}{\sqrt{\lambda+\lambda_{min}(\bm{\Sigma}^{real}_{u,t})}} \leq \frac{\beta'(T^{real}_{u,t},\delta)}{\sqrt{\lambda+T^{real}_{u,t}\tilde{\lambda}_x/8}}
\leq \frac{R\sqrt{d\log(1+\frac{T^{real}_{u,t}}{\lambda d})+2\log\frac{4{\left | \mathcal{U} \right |}}{\delta}}+\sqrt{\lambda}}{\sqrt{\lambda+T^{real}_{u,t}\tilde{\lambda}_x/8}}
\]
Furthermore, if 
\[
T^{real}_{u,t}\geq \frac{512d}{\gamma^2 \tilde{\lambda}_x}\log{\frac{4{\left | \mathcal{U} \right |}}{\delta}}
\]
The following inequality holds
\begin{equation}\label{eq.11}
\frac{R\sqrt{d\log(1+\frac{T^{real}_{u,t}}{\lambda d})+2\log\frac{4{\left | \mathcal{U} \right |}}{\delta}}+\sqrt{\lambda}}{\sqrt{\lambda+T^{real}_{u,t}\tilde{\lambda}_x/8}} \leq \frac{\gamma}{2}
\end{equation}
This is because the Lemma 10. in \citet{li2018online}, and this inequality represents the cluster is correct at this time.\par
According to \citet{li2018online}, when the following event holds, the cluster will be partitioned correctly with probability at least $1-3\delta/4$.
\[
B_4(\delta)=\{T^{real}_{u,t}\geq \max\{\frac{1024}{\tilde{\lambda}_x^2}\log{\frac{512d}{\tilde{\lambda}_x^2 \delta}},\frac{512d}{\gamma^2 \tilde{\lambda}_x}\log{\frac{4{\left | \mathcal{U} \right |}}{\delta}}\}\text{ for all agent }u\in \mathcal{U}\}.
\]
To obtain this, we can take 
\[
T_0(\delta)=16{\left | \mathcal{U} \right |}\log{\frac{4{\left | \mathcal{U} \right |}T}{\delta}}+4{\left | \mathcal{U} \right |}\max\{\frac{1024}{\tilde{\lambda}_x^2}\log{\frac{512d}{\tilde{\lambda}_x^2 \delta}},\frac{512d}{\gamma^2 \tilde{\lambda}_x}\log{\frac{4{\left | \mathcal{U} \right |}}{\delta}}\}
\]
According to Lemma 8 in \citet{li2018online} and \cref{apdx_lem:new_lambda_x}, when $t>T_0(\delta)$, the event above will hold with probability at least $1-\delta/4$. \par
Owing to the fact that the data stored at the server are not the real data, i.e.
$T_{u,t}^{ser} \ne T^{real}_{u,t}$,
we need a additional time steps, i.e. $T_1$, to ensure that every agent can communicate with the server at least once to make the server obtain the real data.\par
Note that in each timestep the activated agent can communicate with the server if satisfying the condition that det($\bm{\Sigma}_{u_{t},t}+\bm{\Sigma}_{u_t,t}^{loc}$) $>(1+\alpha_c)$det($\bm{\Sigma}_{u_{t},t}$) or with probability $p_t=\frac{3\log t}{t}$. The following lemma can help us bound $T_1$\par

\begin{lemma}\label{lemma 1.1}
Under our algorithm, for sufficiently large $T$, beginning with any $T_0=O(\log \frac{T}{\delta})$, each agent will be communicated with server at least once time with probability at least $1-\delta^{\frac{3}{2}}$ within $O(\sqrt{1/\delta})$ activate timesteps.
\end{lemma}
With this Lemma, we can bound that $B_{3u}^c(\delta)\leq \delta^{\frac{3}{2}}$\par
Then we need each agent $u$ can be activated at least $\sqrt{1/\delta}$ times during $T_1$, because of Lemma 8. in \cite{li2018online} we can take
\[
T_1=16{\left | \mathcal{U} \right |}\log\frac{4{\left | \mathcal{U} \right |}T}{\delta}+4{\left | \mathcal{U} \right |}\sqrt{1/\delta}
\]
With this $T_1$, the event
\[
B_5(\delta)=\{T^{real}_{u,T_0 \sim T_0+T_1}\geq \sqrt{1/\delta}\text{ for all agent }u\in \mathcal{U}\}
\]
will hold with probability at least $1-\delta/4$. \par
Finally, we can conclude that when $t>T_0+T_1$
\[
\bigcap_{u=1}^{{\left | \mathcal{U} \right |}}(B_{1u}(\delta/4{\left | \mathcal{U} \right |})B_{2u}(\delta/4{\left | \mathcal{U} \right |})B_{3u}(\delta)B_{4}(\delta)B_{5}(\delta))\subseteq \varepsilon
\]
Then
\[
P(\bigcap_{u=1}^{{\left | \mathcal{U} \right |}}(B_{1u}(\delta/4{\left | \mathcal{U} \right |})B_{2u}(\delta/4{\left | \mathcal{U} \right |})B_{3u}(\delta))B_{4}(\delta)B_{5}(\delta))\leq P(\varepsilon)
\]
Further
\begin{equation}\nonumber
\begin{aligned}
P(\varepsilon^c) 
&\leq 1-P(\bigcap_{u=1}^{{\left | \mathcal{U} \right |}}(B_{1u}(\delta/4{\left | \mathcal{U} \right |})B_{2u}(\delta/4{\left | \mathcal{U} \right |})B_{3u}(\delta))B_{4}(\delta)B_{5}(\delta))\\
&\leq \sum_{u=1}^{{\left | \mathcal{U} \right |}}(P(B_{1u}^c(\delta/4{\left | \mathcal{U} \right |}))+P(B_{2u}^c(\delta/4{\left | \mathcal{U} \right |}))+P(B_{3u}^c(\delta)))+P(B_{4}^c(\delta))+P(B_{5}^c(\delta))\\
&\leq \sum_{u=1}^{{\left | \mathcal{U} \right |}}(\delta/4{\left | \mathcal{U} \right |}+\delta/4{\left | \mathcal{U} \right |}+\delta^{\frac{3}{2}})+\delta/4+\delta/4\\
&=\delta+{\left | \mathcal{U} \right |}\delta^{\frac{3}{2}}
\end{aligned}
\end{equation}
If we take $\delta={\left | \mathcal{U} \right |}^2/T$, then the last term of \cref{eq.10c} will become
\[
(T-T_0-T_1)P(\varepsilon^c)=O({\left | \mathcal{U} \right |}^2+\frac{{\left | \mathcal{U} \right |}^4}{\sqrt{T}})
\]
The equation is because when $\delta={\left | \mathcal{U} \right |}^2/T$,
\[
T_0=16{\left | \mathcal{U} \right |}\log{\frac{4T^2}{{\left | \mathcal{U} \right |}}}+4{\left | \mathcal{U} \right |}\max\{\frac{1024}{\tilde{\lambda}_x^2}\log{\frac{512dT}{\tilde{\lambda}_x^2 {\left | \mathcal{U} \right |}^2}},\frac{512d}{\gamma^2 \tilde{\lambda}_x}\log{\frac{4T}{{\left | \mathcal{U} \right |}}}\}=O({\left | \mathcal{U} \right |}(\frac{d}{\gamma^2\tilde{\lambda}_x}+\frac{1}{\tilde{\lambda}^2_x})\log{T})
\]
\[
T_1=16{\left | \mathcal{U} \right |}\log{\frac{4T^2}{{\left | \mathcal{U} \right |}}}+4\sqrt{T}=O({\left | \mathcal{U} \right |}\log T+\sqrt{T})
\]
Next, we discuss the the term $\mathcal{F}_u^c(\delta)$ and the expected regret $\mathbb{E}[\sum_{t=T_0+T_1+1}^{T}R^{\alpha}_{u_t}(t,\bm{a}^{(u_t)}_{t})|\varepsilon,\mathcal{F}_1(\delta),...,\mathcal{F}_{\left | \mathcal{U} \right |}(\delta)]$.\par
When the clusters are all partitioned correctly at the server, we can see that our algorithm runs the FedLinUCB \cite{he2022simple} independently in each cluster with the $p_t$-auxiliary communication and cascading arms. Each cluster is independent, so we can put the agents which belong to the same cluster together, and reorder the sequence of active agents which belong to the same cluster such that each agent communicates with the server and stay activating until the next agent is communicated with the server, consistent with \citet{he2022simple}. Then we note that the $p_t$-auxiliary communication does not affect the analysis of our regret.\par
In the following, we denote the common communication as the communication caused by the condition det($\bm{\Sigma}_{u_{t},t}+\bm{\Sigma}_{u_t,t}^{loc}$) $>(1+\alpha_c)$det($\bm{\Sigma}_{u_{t},t}$), and the auxiliary communication as the communication that caused by $p_t$-auxiliary condition. 
Because in the work of \citet{he2022simple}, their result is for common communication, we claim that their result also holds with our auxiliary condition as follows.
\begin{lemma}\label{lemma 1}
(Global confidence bound; Theorem 2, \cite{10.5555/2986459.2986717}) In cluster $\mathcal{V}_j$, for any $\delta>0$, with probability at least $1-\delta$, for each $t\in[T]$ satisfies    
\begin{equation}\nonumber
    \left\|\hat{\bm{\theta}}^{all}_{\mathcal{V}_j,t}-\bm{\theta}^*_{\mathcal{V}_j}\right\|_{\bm{\Sigma}_{\mathcal{V}_j,t}^{all}}\leq R\sqrt{2\log(\frac{det(\bm{\Sigma}_{\mathcal{V}_j,t}^{all})^{1/2}det(\lambda I)^{-1/2}}{\delta})}+\sqrt{\lambda}
\end{equation}
and
\begin{equation}\nonumber
    \left\|\bm{g}_{\mathcal{V}_{j(u)}, t}^{all}\right\|_{(\bm{\Sigma}_{\mathcal{V}_{j(u)},t}^{all})^{-1}}\leq R\sqrt{2\log(\frac{det(\bm{\Sigma}_{\mathcal{V}_j,t}^{all})^{1/2}det(\lambda I)^{-1/2}}{\delta})}\leq R\sqrt{d\log(1+\frac{KT}{\lambda d})+2\log(\frac{1}{\delta})}
\end{equation}
\end{lemma}
The inequality above is because 
\[
det(\bm{\Sigma}_{\mathcal{V}_j,t}^{all})\leq (\lambda + T_{\mathcal{V}_j,t}^{all}/d)^d \leq (\lambda + T_{\mathcal{V}_j}^{all}/d)^d \leq (\lambda + KT/d)^d
\]\par
\begin{lemma}\label{lemma 2.1}
(Covariance comparison) It holds that for all $t\in[T]$ and each agent $u\in\mathcal{U}$
\begin{equation}\label{eq.12a}
    \lambda I+\sum_{u'\in \mathcal{V}_{j(u)}}\bm{\Sigma}_{u', t}^{ser} \geq \frac{1}{\alpha_c}\bm{\Sigma}_{u, t}^{loc}
\end{equation}
Moreover, for any $1\leq t_1 < t_2 \leq T$ if agent $u$ is the only active agent from round $t_1$ to $t_2-1$ and agent $u$ only communicates with the server at round $t_1$, then for all $t_1+1\leq t \leq t_2-1$, it further holds that
\begin{equation}\label{eq.12b}
    \bm{\Sigma}_{u,t}\geq \frac{1}{1+|\mathcal{V}_{j(u)}|\alpha_c}\bm{\Sigma}_{\mathcal{V}_{j(u)}, t}^{all}
\end{equation}
In particular, if $u$ is not the only activated agent from timestep $t_1+1$ to $t_2-1$, we have
\begin{equation}\label{eq.12c}
    \bm{\Sigma}_{u,t}\geq \frac{1}{1+|\mathcal{V}_{j(u)}|\alpha_c}\bm{\Sigma}_{\mathcal{V}_{j(u)}, t_1}^{all}
\end{equation}
\end{lemma}
\begin{lemma}\label{lemma 2.2}
    (Local concentration) For any $\delta>0$, with probability at least $1-\delta$, for any $t\in[T],u\in\mathcal{U}$, it holds that
\begin{equation}\label{eq.13}
    \left\|(\alpha_c\lambda I+\bm{\Sigma}_{u, t}^{loc})^{-1}\bm{g}_{u, t}^{loc}\right\|_{\alpha_c\lambda I+\bm{\Sigma}_{u, t}^{loc}}\leq R\sqrt{d\log(1+\frac{KT}{\alpha_c \lambda d})+2\log(\frac{1}{\delta})+2\log(T^2{\left | \mathcal{U} \right |})}
\end{equation}
\end{lemma}
\begin{lemma}\label{lemma 3.1}
(restated \cref{lem:concentration} in the section of Theoretical Analysis; Local confidence bound) Under our algorithm and setting, when the cluster is correct, then with probability at least $1-\delta$, for each $t \in [T]$ and every $u \in \mathcal{U}$ with respect to each cluster $\mathcal{V}_{j(u)}$, it holds that $\left\|\hat{\bm{\theta}}_{u,t}-\bm{\theta}^{j(u)}\right\|_{\bm{\Sigma}_{u,t}}\leq \beta(\delta)$.  
\end{lemma}
\begin{lemma}\label{lemma 3.2}
(Lemma 4.2, \cite{li2016contextual}) When Lemma 3.1 holds, for each item $i$ we have 
\[
0\leq U_t(i)-\bar{w}_t(i)\leq 2\beta(\delta)\left \| \bm{x}_{t,i} \right \|_{(\bm{\Sigma}_{u_{t},t})^{-1}}
\]
\end{lemma}
\cref{lemma 2.1}, \cref{lemma 2.2}, and \cref{lemma 3.1} are inspired by \citet{he2022simple} and we can prove that these lemmas still hold with our setup (see more details in the next sub-section).\par
According to \cref{lemma 3.1}, we can conclude that $P(\mathcal{F}_u(\delta)^c)\leq \delta/{\left | \mathcal{U} \right |}$.\par
Now we bound the item $\mathbb{E}[\sum_{t=T_0+T_1+1}^{T}R^{\alpha}_{u_t}(t,\bm{a}^{(u_t)}_{t})|\varepsilon,\mathcal{F}_1(\delta),...,\mathcal{F}_{\left | \mathcal{U} \right |}(\delta)]$.\par
For each timestep, the per-round regret
\begin{equation}\label{apdx_eq:per-round_regret}
    R^{\alpha}_{u}(t,\bm{a}^{(u)}_{t})=\alpha f^{(u)}(\bm{a}^*_t,\bm{w}_t)-f^{(u)}(\bm{a}_t,\bm{w}_t)
\end{equation}
This equation is because the local agent maintains an $\alpha$-approximate oracle. Without loss of generality, we set $f(\bm{a}, \bm{w}_t) = 1 - \prod_{k=1}^{len(\bm{a})} (1-w_t(a_k))$ and $len(\bm{a}_t)=K$ for all $t \in [T]$. Then at timestep $t$, the reward with respect to the activated agent $u_t$ satisfies
\begin{equation}\nonumber
    \begin{aligned}
        f^{(u_t)}(\bm{a}_t,U_t)
        &\geq\alpha f^{(u_t)}(\bm{a}_t^{*U},U_t)\\
        &\geq\alpha f^{(u_t)}(\bm{a}_t^{*},U_t)\\
        &\geq\alpha f^{(u_t)}(\bm{a}_t^{*},\bar{\bm{w}}_t),
    \end{aligned}
\end{equation}
where $\bm{a}_t^{*U}$ is the optimal reward with respect to the expected weight is $U_t$ and $\bar{\bm{w}}_t=\mathbb{E}[\bm{w}_t]$. The last inequality is because \cref{lemma 3.2} and $f(\bm{a}, \bm{w}_t)$ is a non-decreasing with respect to $\bm{w}_t$. \par
Then the expectation of per-round regret can be bounded by
\begin{equation}\label{eq.14}
    \begin{aligned}
        \mathbb{E}_{w}[R^{\alpha}_{u_t}(t,\bm{a}^{(u_t)}_{t})]
        &=\alpha \mathbb{E}_{w}[f^{(u_t)}(\bm{a}^*_t,\bm{w}_t)]-\mathbb{E}_{w}[f^{(u_t)}(\bm{a}_t,\bm{w}_t)]\\
		&=\alpha f^{(u_t)}(\bm{a}^*_t,\bar{\bm{w}}_t)-f^{(u_t)}(\bm{a}_t,\bar{\bm{w}}_t)\\
        &\leq f^{(u_t)}(\bm{a}_t,U_t)-f^{(u_t)}(\bm{a}_t,\bar{\bm{w}}_t)
    \end{aligned}
\end{equation}

According to the analysis of \cref{thm:main} in \citet{li2018online}, we can prove that
\begin{equation}\label{eq.15}
\begin{aligned}
f^{(u_t)}(\bm{a}_t,U_t)-f^{(u_t)}(\bm{a}_t,\bar{\bm{w}}_t)
&\leq\sum_{k=1}^{K}\prod_{l=1}^{k-1}(1-\bar{w}_t(a_{l,t}))(U_t(a_{k,t})-\bar{w}_t(a_{k,t}))
\end{aligned}
\end{equation}\par
Note that the term $\prod_{l=1}^{k-1}(1-\bar{w}_t(a_{l,t}))$ is the probability that the user check the $k$-th item (whatever the user click or not), we define $p_{t,k}$ to represent this for convenience. Under the lemmas above, we can have
\begin{equation}\label{eq.16}
    \begin{aligned}
        \mathbb{E}[R^{\alpha}_{u_t}(t,\bm{a}^{(u_t)}_{t})]
        &\leq \mathbb{E}[f^{(u_t)}(\bm{a}_t,U_t)-f^{(u_t)}(\bm{a}_t,\bar{\bm{w}}_t)]\\
        &\leq \mathbb{E}[\sum_{k=1}^{K}p_{t,k}(U_t(a_{k,t})-\bar{w}_t(a_{k,t}))]\\
		&\leq 2\beta(\delta)\sum_{k=1}^{K}\mathbb{E}[\left \| \bm{x}_{t,a_{k,t}} \right \|_{(\bm{\Sigma}_{u_{t},t})^{-1}}] 
    \end{aligned}
\end{equation}\par
In our setting with cascading arms, our covariance matrix which is used to calculate the estimated vector only contains the information that has been observed, so the UCB-based bound only guarantees the sum of the first $O_t$ items. It means that \cref{lemma 3.1} and \cref{lemma 3.2} holds when $i \le O_t$. According to \citet{Liu2023ContextualCB}, we can deal with this by using the TPM Bounded Smoothness. As we defined above, we use $p_{t,k}$ to represent the probability of the agent observing the feedback of the first $k$ base arms, i.e. $p_{t,k}=P(O_t=k)$. According to the analysis in \citet{Liu2023ContextualCB}, the expectation of per-round regret, when clusters are correct, is
\begin{equation}\label{eq.16}
    \begin{aligned}
&\mathbb{E}[R^{\alpha}_{u_t}(t,\bm{a}^{(u_t)}_{t})\mid \varepsilon,\mathcal{F}_1(\delta),...,\mathcal{F}_{\left | \mathcal{U} \right |}(\delta),\mathcal{H}_{t-1}]\\
&\leq \mathbb{E}[\sum_{k=1}^{K}p_{t,k}(U_t(a_{k,t})-\bar{w}_t(a_{k,t}))\mid \varepsilon,\mathcal{F}_1(\delta),...,\mathcal{F}_{\left | \mathcal{U} \right |}(\delta),\mathcal{H}_{t-1}]\\
&=\mathbb{E}[\sum_{k=1}^{K}\mathbb{E}[\mathbb{I}\{O_t=k\} \mid \mathcal{H}_{t-1}](U_t(a_{k,t})-\bar{w}_t(a_{k,t}))\mid \varepsilon,\mathcal{F}_1(\delta),...,\mathcal{F}_{\left | \mathcal{U} \right |}(\delta),\mathcal{H}_{t-1}]\\
&=\mathbb{E}[\sum_{k=1}^{K}\mathbb{I}\{O_t=k\}(U_t(a_{k,t})-\bar{w}_t(a_{k,t}))\mid \varepsilon,\mathcal{F}_1(\delta),...,\mathcal{F}_{\left | \mathcal{U} \right |}(\delta),\mathcal{H}_{t-1}]\\
&=\mathbb{E}[\sum_{k=1}^{O_t}(U_t(a_{k,t})-\bar{w}_t(a_{k,t}))\mid \varepsilon,\mathcal{F}_1(\delta),...,\mathcal{F}_{\left | \mathcal{U} \right |}(\delta),\mathcal{H}_{t-1}]\\
&\leq \mathbb{E}[2\beta(\delta)\sum_{k=1}^{O_t}\left \| \bm{x}_{t,a_{k,t}} \right \|_{(\bm{\Sigma}_{u_{t},t})^{-1}}\mid \varepsilon,\mathcal{F}_1(\delta),...,\mathcal{F}_{\left | \mathcal{U} \right |}(\delta),\mathcal{H}_{t-1}]
    \end{aligned}
\end{equation}\par
We denote the expectation of per-round regret as $2\mathbb{E}_{t}[\beta(\delta)\sum_{k=1}^{O_t}\left \| \bm{x}_{t,a_{k,t}} \right \|_{(\bm{\Sigma}_{u_{t},t})^{-1}}]$ for simplicity, and our subsequent analysis is under the case that the clusters are correct, so we ignore the condition $\varepsilon,\mathcal{F}_1(\delta),...,\mathcal{F}_{\left | \mathcal{U} \right |}(\delta)$ for simplicity. Then we can get
\begin{equation}\label{eq.17}
    \begin{aligned}
\mathbb{E}[\sum_{t=T_0+T_1+1}^{T}R^{\alpha}_{u_t}(t,\bm{a}^{(u_t)}_{t})]\leq 2\beta(\delta)\mathbb{E}[\sum_{t=T_0+T_1+1}^{T}\sum_{k=1}^{O_t}\left \| \bm{x}_{t,a_{k,t}} \right \|_{(\bm{\Sigma}_{u_{t},t})^{-1}}]
    \end{aligned}
\end{equation}
We focus on the term $\sum_{t=T_0+T_1+1}^{T}\sum_{k=1}^{O_t}\left \| \bm{x}_{t,a_{k,t}} \right \|_{(\bm{\Sigma}_{u_{t},t})^{-1}}$ under correct cluster and we use $\sum_1^{T'}$ instead of $\sum_{t=T_0+T_1+1}^{T}$ without loss of generality. \par
We use $T_{\mathcal{V}_j}$ different from $T_{\mathcal{V}_j}^{all}$ to denote the total activate times (not the arm play) of cluster $\mathcal{V}_j$. Because after every clusters are correct, each different cluster is independent, we can reorder the agents which belong to the same cluster together, i.e. the activated agents during $T_{\mathcal{V}_j}$ all belong to cluster $\mathcal{V}_j$. In the following, we consider one cluster, omit others for simplicity, and put them back together to get the regret bound in the end. With this reordering, we consider cluster $\mathcal{V}_1$ with timesteps from $1$ to $T_{\mathcal{V}_j}$ without loss of generality. If any agent $u \in \mathcal{V}_j$ was activated, we say cluster $\mathcal{V}_j$ was activated. In a similar way, we say the cluster $\mathcal{V}_j$ was communicated with server if any agent $u \in \mathcal{V}_j$ was communicated with server.\par
% The following analysis is under our reorder and in cluster $\mathcal{V}_1$
Note that we have $\sum_1^{T'}=\sum_{j=1}^m T_{\mathcal{V}_j}$ and $T_{\mathcal{V}_j}^{all} \leq KT_{\mathcal{V}_j}$, the inequality is because $O_t \leq len(\bm{a}_t) \leq K$ and we set $len(\bm{a}_t)=K$ in our model. Inside each cluster our algorithm will behave like FedLinUCB with cascading arms \cite{he2022simple}, the estimated vector and covariance will not change until this agent communicates with the server based on auxiliary communication or common communication. In each $T_{\mathcal{V}_j}$, we can reorder the sequence of activating agents which belong to $\mathcal{V}_j$. This is because for each agent the estimated vector and covariance will not change until this agent has communicated with the server caused by auxiliary communication or common communication. In other words, for an agent $u$, assume that $u$ has communicated with server at timestep $t_{u_1}$ and $t_{u_2}$, and $u$ was activated at timestep $t_{u_1}^{'}$ where $t_{u_1}< t_{u_1}^{'}< t_{u_2}$ (note that at $t_{u_1}^{'}$, agent $u$ has not communicated with server), assume that there are timesteps $t_{u_2}^{'} \ne t_{u_1}^{'}$, then if we put agent $u$ in $t_{u_2}^{'}$ just like agent was activated at $t_{u_2}^{'}$ instead of at $t_{u_1}^{'}$, we can see that the progress of agent $u$ was activated at $t_{u_1}^{'}$ is equal to the progress of agent $u$ was activated at $t_{u_2}^{'}$. Since all agents are independent, according to the analysis above and inspired by the work of \citet{he2022simple}, we can also reorder the agent in the same cluster such that when an agent communicates with the server, then this agent will be the only activated agent until the next communication happen. Using formal math to express this is that for the total communication times of $N$, and the communication happens at $t_1,...,t_N$, where $0 = t_0 <t_1...<t_N = T_{\mathcal{V}_j}+1$, for $t_i \leq t < t_{i+1}, i \in [N]$, the activated agent is same, i.e. when the agent $u_{t_i}$ is communicated with the server at timestep $t_i$ for $i \in [N]$, then it will have $u_{t_i}=u_{t_i+1}=...=u_{t_{i+1}-1}$ in the following timestep $t_i \leq t < t_{i+1}$. Note that with the above discussion, we only consider the agents which belong to the same cluster $\mathcal{V}_j$, i.e. $u_{t_i} \in \mathcal{V}_j$ for all $0 \leq i \leq N$.\par
According to \cref{eq.16}
\begin{equation}\label{eq.18}
    \begin{aligned}
\mathbb{E}_{t}[\sum_{t=1}^{T_{\mathcal{V}_1}}R^{\alpha}_{u_t}(t,\bm{a}^{(u_t)}_{t})]\leq 2\beta(\delta)\mathbb{E}_{t}[\sum_{t=1}^{T_{\mathcal{V}_1}}\sum_{k=1}^{O_t}\left \| \bm{x}_{t,a_{k,t}} \right \|_{(\bm{\Sigma}_{u_{t},t})^{-1}}]
    \end{aligned}
\end{equation}\par
And under our reordering
\begin{equation}\label{eq.19}
\sum_{t=1}^{T_{\mathcal{V}_1}}\sum_{k=1}^{O_t}\left \| \bm{x}_{t,a_{k,t}} \right \|_{(\bm{\Sigma}_{u_{t},t})^{-1}}=\sum_{i=0}^{N}\sum_{t=t_i+1}^{t_{i+1}-1}\sum_{k=1}^{O_t}\left \| \bm{x}_{t,a_{k,t}} \right \|_{(\bm{\Sigma}_{u_{t},t})^{-1}}+\sum_{i=0}^{N}\max\{\sum_{k=1}^{O_{t_i}}\left \| \bm{x}_{t,a_{k,t_i}} \right \|_{(\bm{\Sigma}_{u_{t_i},t_i})^{-1}},K\}
\end{equation}\par
The second term on the right is because the number of the communication of some agent may be small so it can not satisfy our following analysis. However, we bound this case trivially by $K$ for their timesteps, and with this consideration, we can ignore the term $2\beta(\delta)$ and $p_{t,k}$ in these timesteps. For convenience, we will still take $2\beta(\delta)$ for all timesteps and remove it in the final.\par
For each timestep $t_i+1 \leq t \leq t_{i+1}-1$, according to Lemma 2.1 we have
\begin{equation}\label{eq.20}
\sum_{i=0}^{N}\sum_{t=t_i+1}^{t_{i+1}-1}\sum_{k=1}^{O_t}\left \| \bm{x}_{t,a_{k,t}} \right \|_{(\bm{\Sigma}_{u_{t},t})^{-1}}\leq \sqrt{1+|\mathcal{V}_1|\alpha_c}\sum_{i=0}^{N}\sum_{t=t_i+1}^{t_{i+1}-1}\sum_{k=1}^{O_t}\left \| \bm{x}_{t,a_{k,t}} \right \|_{(\bm{\Sigma}_{\mathcal{V}_1,t}^{all})^{-1}}
\end{equation}
Next we need to show that for $t=t_i,i \in [N]$ it still holds.\par
To do this, we partition $T_{\mathcal{V}_1}$ in the following way and define
\[
T_i=\min\{t \in [T_{\mathcal{V}_1}]| \text{det}(\bm{\Sigma}_{\mathcal{V}_1,t}^{all})\geq 2^i \lambda^d]\}
\]
For agent $u \in [\mathcal{V}_1]$, we first consider the case for some interval $T_i$ to $T_{i+1}$ where $u$ communicates with the server more than once. In other words, agent $u$ communicates with the server at $T_{i,1}$ and $T_{i,2}$ so that $T_i \leq T_{i,1}<T_{i,2}<T_{i+1}$. Then for $T_{i,2}$, we have:
\begin{equation}\label{eq.21}
    \begin{aligned}
\left \| \bm{x}_{T_{i,2},a_{k,T_{i,2}}} \right \|_{\bm{\Sigma}_{u,T_{i,2}}^{-1}}
&\leq \sqrt{1+|\mathcal{V}_1|\alpha_c}\left \| \bm{x}_{T_{i,2},a_{k,T_{i,2}}} \right \|_{(\bm{\Sigma}_{\mathcal{V}_1,T_{i,1}}^{all})^{-1}}\\
&\leq \sqrt{2(1+|\mathcal{V}_1|\alpha_c)}\left \| \bm{x}_{T_{i,2},a_{k,T_{i,2}}} \right \|_{(\bm{\Sigma}_{\mathcal{V}_1,T_{i+1}-1}^{all})^{-1}}\\
&\leq \sqrt{2(1+|\mathcal{V}_1|\alpha_c)}\left \| \bm{x}_{T_{i,2},a_{k,T_{i,2}}} \right \|_{(\bm{\Sigma}_{\mathcal{V}_1,T_{i,2}}^{all})^{-1}},
    \end{aligned}
\end{equation}
where the first inequality is because \cref{lemma 2.1}, the second inequality is because $\text{det}(\bm{\Sigma}_{\mathcal{V}_1,T_{i+1}-1}^{all})/\text{det}(\bm{\Sigma}_{\mathcal{V}_1,T_{i,1}}^{all}) \leq (2^{i+1}\lambda^d)/(2^{i}\lambda^d)=2$, and the last inequality is because $\bm{\Sigma}_{\mathcal{V}_1,T_{i+1}-1}^{all} \geq \bm{\Sigma}_{\mathcal{V}_1,T_{i,2}}^{all}$.\par
At timestep $T_{i,1}$, as we mentioned above, we can bound it by $K$ trivially, so we only need to bound the number of the interval $T_i$. According to Determinant-Trace Inequality \cite{abbasi2011improved}, we have
\[
\text{det}(\bm{\Sigma}_{\mathcal{V}_1,T_{\mathcal{V}_1}}^{all})\leq (\lambda + T_{\mathcal{V}_1}^{all}/d)^d \leq (\lambda + KT_{\mathcal{V}_1}/d)^d \leq (\lambda + KT_{\mathcal{V}_1})^d
\]
which implies that the number of different intervals is at most $dK\log_2(1+ KT_{\mathcal{V}_1}/\lambda)$ for each agents $u \in \mathcal{V}_1$, so the total regret for all agents in cluster $\mathcal{V}_1$ at $T_{i,1}$ is bounded by $d|\mathcal{V}_1|K\log_2(1+ KT_{\mathcal{V}_1}/\lambda)$. Combining the inequality above, we can bound the term in cluster $\mathcal{V}_1$ as
\begin{equation}\label{eq.22}
    \begin{aligned}
\sum_{t=1}^{T_{\mathcal{V}_1}}\sum_{k=1}^{O_t}\left \| \bm{x}_{t,a_{k,t}} \right \|_{(\bm{\Sigma}_{u_{t},t})^{-1}} \leq d|\mathcal{V}_1|K\log_2(1+ KT_{\mathcal{V}_1}/\lambda)+\sqrt{2(1+|\mathcal{V}_1|\alpha_c)}\sum_{t=1}^{T_{\mathcal{V}_1}}\sum_{k=1}^{O_t}\left \| \bm{x}_{t,a_{k,t}} \right \|_{(\bm{\Sigma}_{\mathcal{V}_1,t}^{all})^{-1}}
    \end{aligned}
\end{equation}
According to Lemma A.4 in \citet{li2016contextual}
\begin{equation}\label{eq.23}
    \begin{aligned}
\sum_{t=1}^{T_{\mathcal{V}_1}}\sum_{k=1}^{O_t}\left \| \bm{x}_{t,a_{k,t}} \right \|_{(\bm{\Sigma}_{\mathcal{V}_1,t}^{all})^{-1}}
&\leq \sqrt{\sum_{t=1}^{T_{\mathcal{V}_1}}O_t \cdot \sum_{t=1}^{T_{\mathcal{V}_1}}\sum_{k=1}^{O_t}\left \| \bm{x}_{t,a_{k,t}} \right \|_{(\bm{\Sigma}_{\mathcal{V}_1,t}^{all})^{-1}}^2}\\
&=\sqrt{T_{\mathcal{V}_1}^{all} \cdot \sum_{t=1}^{T_{\mathcal{V}_1}}\sum_{k=1}^{O_t}\left \| \bm{x}_{t,a_{k,t}} \right \|_{(\bm{\Sigma}_{\mathcal{V}_1,t}^{all})^{-1}}^2}\\
&\leq \sqrt{KT_{\mathcal{V}_1} \cdot 2\log(\frac{\text{det}(\bm{\Sigma}_{\mathcal{V}_1,t}^{all})}{\lambda^d})}\\
&\leq \sqrt{2dKT_{\mathcal{V}_1}\log(1+\frac{KT_{\mathcal{V}_1}}{\lambda})}
    \end{aligned}
\end{equation}
Plugging this inequality into \cref{eq.22} we have 
\begin{equation}\label{eq.24}
    \begin{aligned}
\sum_{t=1}^{T_{\mathcal{V}_1}}\sum_{k=1}^{O_t}\left \| \bm{x}_{t,a_{k,t}} \right \|_{(\bm{\Sigma}_{u_{t},t})^{-1}} \leq d|\mathcal{V}_1|K\log_2(1+ KT_{\mathcal{V}_1}/\lambda)+2\sqrt{1+|\mathcal{V}_1|\alpha_c}\cdot \sqrt{dKT_{\mathcal{V}_1}\log(1+\frac{KT_{\mathcal{V}_1}}{\lambda})}
    \end{aligned}
\end{equation}
The other clusters' analysis the same as that for cluster $\mathcal{V}_1$, so the total bounded by
\begin{equation}\label{eq.25}
    \begin{aligned}
\mathbb{E}[\sum_{t=1}^{T'}\sum_{k=1}^{O_t}\left \| \bm{x}_{t,a_{k,t}} \right \|_{(\bm{\Sigma}_{u_{t},t})^{-1}}]
&\leq \sum_{j=1}^{J}\left(d|\mathcal{V}_j|K\log_2(1+ KT_{\mathcal{V}_j}/\lambda)+2\sqrt{1+|\mathcal{V}_j|\alpha_c}\cdot \sqrt{dKT_{\mathcal{V}_j}\log(1+\frac{KT_{\mathcal{V}_j}}{\lambda})}\right)\\
&\leq dK\sum_{j=1}^{J}|\mathcal{V}_j|\log_2(1+ KT/\lambda)+2\sum_{j=1}^{J}\sqrt{T_{\mathcal{V}_j}(1+|\mathcal{V}_j|\alpha_c)}\sqrt{dK\log(1+KT/\lambda)}\\
&\leq dK{\left | \mathcal{U} \right |}\log(1+ KT/\lambda)+2\sqrt{\sum_{j=1}^{J}T_{\mathcal{V}_j}\cdot\sum_{j=1}^{J}(1+|\mathcal{V}_j|\alpha_c)}\sqrt{dK\log(1+KT/\lambda)}\\
&\leq dK{\left | \mathcal{U} \right |}\log(1+ KT/\lambda)+2\sqrt{J+{\left | \mathcal{U} \right |}\alpha_c}\sqrt{dKT\log(1+KT/\lambda)}
    \end{aligned}
\end{equation}
The above equation and inequality is because the mean inequality and the fact that $\sum_{j=1}^{J}T_{\mathcal{V}_j}=T'=T-T_0-T_1\leq T$ and $\sum_{j=1}^{J}|\mathcal{V}_j|={\left | \mathcal{U} \right |}$. Finally, if we take ${\left | \mathcal{U} \right |}<O(T^{1/4})$, we sum all the items together to get our main result:
\begin{equation}\nonumber
    \begin{aligned}
\begin{split}
Regret(T)\leq 2dK{\left | \mathcal{U} \right |}\log(1+ KT/\lambda)+4\beta(\delta)\sqrt{J+{\left | \mathcal{U} \right |}\alpha_c}\sqrt{dKT\log(1+KT/\lambda)}\\+O({\left | \mathcal{U} \right |}(\frac{d}{\gamma^2\tilde{\lambda}_x}+\frac{1}{\tilde{\lambda}^2_x})\log{T}+\sqrt{T})
\end{split}
    \end{aligned}
\end{equation}
Thus we finish the proof of \cref{thm:main}
\end{proof}
\subsection{Proof for the Technical Lemmas}
\begin{lemma}\label{apdx_lem:new_lambda_x}
    \label{assumption}
    Under Assumption \ref{assumption3}, at any time $t$, for any fixed unit vector $\btheta \in \R^d$
    \begin{equation}
        \mathbb{E}_t[(\btheta^{\top}\bx_{a_t})^2\mid\mathcal{I}_t]\geq\tilde{\lambda}_x\triangleq\int_{0}^{\lambda_x} (1-e^{-\frac{(\lambda_x-x)^2}{2\sigma^2}})^{I} dx\,.
    \end{equation}
\end{lemma}
\begin{proof}[Proof of \cref{apdx_lem:new_lambda_x}]
The proof of this lemma mainly follows the proof of Claim 1 in \cite{gentile2014online}, but with more careful analysis, since their assumption is more stringent than ours.

Denote the feasible arms at round $t$ by ${\mathcal{I}_t=\{\bx_{t,1},\bx_{t,2},\ldots,\bx_{t,\mathcal{I}_t}\}}$. 
Consider the corresponding i.i.d. random variables $\theta_i=(\btheta^{\top}\bx_{t,i})^2-\mathbb{E}_t[(\btheta^{\top}\bx_{t,i})^2|\mathcal{I}_t], i=1,2,\ldots,\mathcal{I}_t$. By Assumption \ref{assumption3}, $\theta_i$ s are sub-Gaussian random variables with variance bounded by $\sigma^2$. Therefore, we have that for any $\alpha>0$ and any $i\in[\mathcal{I}_t]$:
\begin{equation*}
    \mathbb{P}_t(\theta_i<-\alpha|\mathcal{I}_t)\leq e^{-\frac{\alpha^2}{2\sigma^2}}\,,
\end{equation*}
where $\mathbb{P}_t(\cdot)$ is the shorthand for the conditional probability $\mathbb{P}(\cdot|(i_1,\mathcal{A}_1,r_1),\ldots,(i_{t-1},\mathcal{A}_{t-1},r_{t-1}),i_t)$.

We also have that
$\mathbb{E}_t[(\btheta^{\top}\bx_{t,i})^2|\mathcal{I}_t=\mathbb{E}_t[\btheta^{\top}\bx_{t,i}\bx_{t,i}^{\top}\btheta|\mathcal{I}_t]\geq\lambda_{\text{min}}(\mathbb{E}_{\bx\sim \rho}[\bx\bx^{\top}])\geq\lambda_x$ by Assumption \ref{assumption3}.
With the above inequalities, we can get
\begin{equation*}
    \mathbb{P}_t(\min_{i=1,\ldots,\mathcal{I}_t}(\btheta^{\top}\bx_{t,i})^2\geq \lambda_x-\alpha|\mathcal{I}_t)\geq (1-e^{-\frac{\alpha^2}{2\sigma^2}})^I\,,
\end{equation*}
where $I$ is the upper bound of $\mathcal{I}_t$.

Therefore, we have
\begin{align}
\mathbb{E}_t[(\btheta^{\top}\bx_{a_t})^2|\mathcal{I}_t]
&\geq\mathbb{E}_t[\min_{i=1,\ldots,\mathcal{I}_t}(\btheta^{\top}\bx_{t,i})^2|\mathcal{I}_t]\notag\\
&\geq\int_{0}^{\infty} \mathbb{P}_t (\min_{i=1,\ldots,\mathcal{I}_t}(\btheta^{\top}\bx_{t,i})^2\geq x|\mathcal{I}_t) dx\notag\\
&\geq \int_{0}^{\lambda_x} (1-e^{-\frac{(\lambda_x-x)^2}{2\sigma^2}})^{I} dx\triangleq\tilde{\lambda}_x\notag
\end{align}
\end{proof}
\noindent\cref{lemma 1.1} (Restated). Under our \cref{alg:main} , for sufficient large enough $T$, beginning with any $T_0=O(\log \frac{T}{\delta})$, each agent will be communicated with server at least once time with probability at least $1-\delta^{\frac{3}{2}}$ with in $O(\sqrt{1/\delta})$ activate timesteps.\par
\begin{proof}[\textit{Proof of \cref{lemma 1.1}}]
We define the event 
\[
A_t^u=\{\text{the activated agent }u \text{ don't communicate with server at timestep }t\}
\]
\[
B_t^u=\{\text{the random number } p \text{ at timestep } t \text{ for the activated agent }u  \text{ satisfy } p>p_t\}
\]
\[
C_t^u=\{\text{the activated agent } u \text{ satisfy det} (\bm{\Sigma}_{u_{t},t}+\bm{\Sigma}_{u_t,t}^{loc}) <(1+\alpha_c)\text{det}(\bm{\Sigma}_{u_{t},t})\text{ at this timestep }t\}
\]
It is easy to show that $P(B_t^u)=1-p_t$. According to the condition of communication, we have
\[
P(A_t^u)=P(B_t^u\cap C_t^u)
\]
Note that $C_t^u$ is not independent for each $t$ and $u$, so $A_t^u$ is not independent, but we can bound $A_t^u$ by $B_t^u$ which is independent for every timestep $t$ and every agent $u$.
Then the probability that agent $u$ do not communicate with server during $\sqrt{T}$ start from $T_0$ is
\[
P(\bigcap_{t=T_0+1}^{T_0+\sqrt{1/\delta}} A_t^u)=P((\bigcap_{t=T_0+1}^{T_0+\sqrt{1/\delta}} B_t^u) \bigcap (\bigcap_{t=T_0+1}^{T_0+\sqrt{1/\delta}} C_t^u)) \leq P(\bigcap_{t=T_0+1}^{T_0+\sqrt{1/\delta}} B_t^u)=\prod_{t=T_0+1}^{T_0+\sqrt{1/\delta}}(1-p_t)
\]
The inequality is because $P(AB)\leq P(B)$.\par
On the other hand, the function $p_t=\frac{3\log t}{t}$ is monotone decreasing and have $0<\frac{3\log t}{t}<\frac{3}{e}$ when $t>e$. Then the equation above can be bound by
\begin{equation}\label{eq.26}
 \prod_{t=T_0+1}^{T_0+\sqrt{1/\delta}}(1-p_t) \leq (1-\epsilon_{T_0+\sqrt{1/\delta}})^{\sqrt{1/\delta}}  
\end{equation}
According to the analysis above, we know that $T_0=O(\log \frac{ T}{\delta})$. Plugging this into \cref{eq.26}, and according to \cref{lemma 1.2}, if we take $\delta=O(1/T)$, we have
\[
(1-\epsilon_{T_0+\sqrt{1/\delta}})^{\sqrt{1/\delta}}=O(\delta^{\frac{3}{2}}),
\]
which finishes the proof of this lemma. 
\end{proof}
\begin{lemma}\label{lemma 1.2}
    For sufficient large enough $t$, for any $k_1>0$ and $k_2>0$ we have $(1-\frac{k_1\log(k_2\log t+t)}{k_2\log t+t})^t=O(\frac{1}{t^{k_1}})$. In the same time:
\[
\lim_{t \to \infty} \frac{(1-\frac{k_1\log(k_2\log t+t)}{k_2\log t+t})^t}{\frac{1}{t^{k_1}}}=1
\]
\end{lemma}
\begin{proof}[\textit{Proof of \cref{lemma 1.2}}]
 Firstly, it holds that $\lim_{t \to \infty}\frac{\log t}{t}=0$ and $\lim_{t \to 0}\frac{\log(1+t)}{t}=1$
\begin{equation}\nonumber
\begin{aligned}
\lim_{t \to \infty} \frac{(1-\frac{k_1\log(k_2\log t+t)}{k_2\log t+t})^t}{\frac{1}{t^{k_1}}}
&=\lim_{t \to \infty} \frac{e^{t\log(1-\frac{k_1\log(k_2\log t+t)}{k_2\log t+t})}}{\frac{1}{t^{k_1}}}\\
&\overset{x=\frac{1}{t}}{\rightarrow} \lim_{x \to 0^+} \frac{e^{\frac{1}{x}\log(1-\frac{k_1\log(1/x-k_2\log x)}{1/x-k_2\log x})}}{x^{k_1}}\\
&=\lim_{x \to 0^+} \frac{e^{\frac{1}{x}\cdot(-\frac{k_1\log(1/x-k_2\log x)}{1/x-k_2\log x})}}{x^{k_1}}\\
&=\lim_{x \to 0^+} \frac{e^{\frac{-k_1\log(1/x-k_2\log x)}{1-k_2x\log x}}}{x^{k_1}}\\
&=\lim_{x \to 0^+} \frac{e^{\frac{k_1}{1-k_2 x\log x}\log \frac{x}{1-k_2 x\log x}}}{x^{k_1}}\\
&=\lim_{x \to 0^+} \frac{(\frac{x}{1-k_2 x\log x})^{\frac{k_1}{1-k_2 x\log x}}}{x^{k_1}}\\
&=\lim_{x \to 0^+} x^{\frac{k_1}{1-k_2 x\log x}-k_1}\\
&=\lim_{x \to 0^+} e^{\frac{k_1 k_2 x\log^2 x}{1-k_2 x\log x}}\\
&=\lim_{x \to 0^+} e^{\frac{0}{1-0}}\\
&=1
\end{aligned}
\end{equation}
The last equation is because $\lim_{x \to 0}x\log^2 x=0$ and $\lim_{x \to 0}x\log x=0$. Thus we finish the proof. 
\end{proof}
\noindent\cref{lemma 2.1}. (Covariance comparison, restated) Under the assumption above, it hold that for all $t\in[T]$ and each agent $u\in\mathcal{U}$
\begin{equation}\nonumber
    \lambda I+\sum_{u'\in \mathcal{V}_{j(u)}}\bm{\Sigma}_{u', t}^{ser} \geq \frac{1}{\alpha_c}\bm{\Sigma}_{u, t}^{loc}
\end{equation}
Moreover, for any $1\leq t_1 < t_2 \leq T$ if agent $u$ is the only active agent from round $t_1$ to $t_2-1$ and agent $u$ only communicates with the server at round $t_1$, then for all $t_1+1\leq t \leq t_2-1$, it further holds that
\begin{equation}\nonumber
    \bm{\Sigma}_{u,t}\geq \frac{1}{1+|\mathcal{V}_{j(u)}|\alpha_c}\bm{\Sigma}_{\mathcal{V}_{j(u)}, t}^{all}
\end{equation}
In particularly, if $u$ is not the only activated agent from timestep $t_1+1$ to $t_2-1$, we have
\begin{equation}\nonumber
    \bm{\Sigma}_{u,t}\geq \frac{1}{1+|\mathcal{V}_{j(u)}|\alpha_c}\bm{\Sigma}_{\mathcal{V}_{j(u)}, t_1}^{all}
\end{equation}\par
\begin{proof}[\textit{Proof of \cref{lemma 2.1}}]
 Different from \citet{he2022simple} where a single arm is selected and no auxiliary communication is used, we prove the covariance comparison with cascading arms and auxiliary communication.
Let $t_1 \leq t$ denote the last timestep such that agent $u$ is active at $t_1$. If at $t_1$ the agent communicates with the server triggered by auxiliary communication or common communication, it will set $\bm{\Sigma}_{u, t}^{loc}$ as $0$. Note that agent $u$ is not activated from $t_1$ to $t$, then
\[
\lambda I+\sum_{u'\in \mathcal{V}_{j(u)}}\bm{\Sigma}_{u', t}^{ser} \geq 0=\frac{1}{\alpha_c}\bm{\Sigma}_{u, t}^{loc}
\]
If at $t_1$ there is no communication, then according to proof in \citet{he2022simple} it will have
\[
\lambda I+\sum_{u'\in \mathcal{V}_{j(u)}}\bm{\Sigma}_{u', t_1}^{ser} \geq \bm{\Sigma}_{u, t_1} \geq \frac{1}{\alpha_c}\bm{\Sigma}_{u, t_1}^{loc}
\]
Because all the $\bm{\Sigma}$ is nondecreasing with respect to timestep $t$, then at $t \geq t_1$
\[
\lambda I+\sum_{u'\in \mathcal{V}_{j(u)}}\bm{\Sigma}_{u', t}^{ser} \geq \lambda I+\sum_{u'\in \mathcal{V}_{j(u)}}\bm{\Sigma}_{u', t_1}^{ser} \geq \frac{1}{\alpha_c}\bm{\Sigma}_{u, t_1}^{loc} = \frac{1}{\alpha_c}\bm{\Sigma}_{u, t}^{loc}
\]
Thus we prove the first claim. \par
Then for any $1\leq t_1 < t_2 \leq T$ if agent $u$ is the only active agent from round $t_1$ to $t_2-1$ and agent $u$ only communicates with the server at round $t_1$, then due to the above inequality
\[
\lambda I+\sum_{u'\in \mathcal{V}_{j(u)}}\bm{\Sigma}_{u', t}^{ser} \geq \frac{1}{\alpha_c}\bm{\Sigma}_{u, t}^{loc}
\]
We sum all all the agents in cluster $\mathcal{V}_{j(u)}$ 
\[
\lambda I+\sum_{u'\in \mathcal{V}_{j(u)}}\bm{\Sigma}_{u', t}^{ser} \geq \frac{1}{|\mathcal{V}_{j(u)}|\alpha_c}\sum_{u'\in \mathcal{V}_{j(u)}}\bm{\Sigma}_{u', t}^{loc}
\] 
For all $t_1+1\leq t \leq t_2-1$ we have 
\[
\bm{\Sigma}_{u,t}=\bm{\Sigma}_{u,t_1}=\lambda I+\sum_{u'\in \mathcal{V}_{j(u)}}\bm{\Sigma}_{u', t_1}^{ser}=\lambda I+\sum_{u'\in \mathcal{V}_{j(u)}}\bm{\Sigma}_{u', t}^{ser}
\]
The first equation is because agent $u$ only communicate at $t_1$ and the last is because agent $u$ is the only activated one. Then
\begin{equation}\nonumber
\begin{aligned}
\bm{\Sigma}_{u,t} 
&\geq \frac{1}{1+|\mathcal{V}_{j(u)}|\alpha_c}\left (\lambda I+\sum_{u'\in \mathcal{V}_{j(u)}}\bm{\Sigma}_{u', t}^{ser}+ \sum_{u'\in \mathcal{V}_{j(u)}}\bm{\Sigma}_{u', t}^{loc} \right )\\
&= \frac{1}{1+|\mathcal{V}_{j(u)}|\alpha_c}\left ( \lambda I + \sum_{u'\in \mathcal{V}_{j(u)}}\bm{\Sigma}_{u', t}^{real} \right )\\
&=\frac{1}{1+|\mathcal{V}_{j(u)}|\alpha_c}\bm{\Sigma}_{\mathcal{V}_{j(u)}, t}^{all}
\end{aligned}
\end{equation}
With the analysis above, the last inequality of the lemma is obvious. Thus we prove this Lemma. 
\end{proof}
\noindent\cref{lemma 2.2}.
(Local concentration, restated) For any $\delta>0$, with probability at least $1-\delta$, for any $t\in[T],u\in\mathcal{U}$, it holds that
\begin{equation}\label{eq.27}
    \left\|(\alpha_c\lambda I+\bm{\Sigma}_{u, t}^{loc})^{-1}\bm{g}_{u, t}^{loc}\right\|_{\alpha_c\lambda I+\bm{\Sigma}_{u, t}^{loc}}\leq R\sqrt{d\log(1+\frac{KT}{\alpha_c \lambda d})+2\log(\frac{1}{\delta})+2\log(T^2{\left | \mathcal{U} \right |})}
\end{equation}
\begin{proof}[\textit{Proof of \cref{lemma 2.2}}]
According to Lemma 6.4 in \citet{he2022simple}, for any $1\leq t_1 < t_2 \leq T$ and each agent $u\in\mathcal{U}$, we define
\[
\bm{\Sigma}_{u,t_1,t_2}=\alpha_c\lambda I+\sum_{i=t_1+1}^{t_2}\mathbb{I}\{u_i=u\} \sum_{k=1}^{O_i^{(u)}} \bm{x}_{i, a_{k, i}} \bm{x}_{i, a_{k, i}}^T \quad \bm{g}_{u,t_1,t_2}=\sum_{i=t_1+1}^{t_2} \mathbb{I}\{u_i=u\}\sum_{k=1}^{O_i^{(u)}} \eta_{i,a_{k,i}}^{(u)}\bm{x}_{i, a_{k, i}} 
\]
By Theorem 1 in \citet{10.5555/2986459.2986717}, if we fix this agent $u$ and the timestep $t_1,t_2$, then  with probability at least $1-\delta/(T^2{\left | \mathcal{U} \right |})$, we have
\begin{equation}\label{eq.28}
\left\|(\bm{\Sigma}_{u,t_1,t_2})^{-1}\bm{g}_{u,t_1,t_2}\right\|_{\bm{\Sigma}_{u,t_1,t_2}}=\left\|\bm{g}_{u,t_1,t_2}\right\|_{(\bm{\Sigma}_{u,t_1,t_2})^{-1}}\leq R\sqrt{2\log(\frac{det(\bm{\Sigma}_{u,t_1,t_2})^{1/2}det(\alpha_c\lambda I)^{-1/2}}{\delta})+2\log(T^2{\left | \mathcal{U} \right |})}
\end{equation}
According to Lemma A.5 in \citet{li2016contextual}
\[
det(\bm{\Sigma}_{u,t_1,t_2}) \leq (\alpha_c \lambda+K(t_2-t_1)/d)^d \leq (\alpha_c \lambda+KT/d)^d
\]
Plugging this into \cref{eq.28}, we have
\[
\left\|\bm{\Sigma}_{u,t_1,t_2}^{-1}\bm{g}_{u,t_1,t_2}\right\|_{\bm{\Sigma}_{u,t_1,t_2}}\leq R\sqrt{d\log(1+\frac{KT}{\alpha_c \lambda d})+2\log(\frac{1}{\delta})+2\log(T^2{\left | \mathcal{U} \right |})}
\]
Let
\[\beta'(\delta)=R\sqrt{d\log(1+\frac{KT}{\alpha_c \lambda d})+2\log(\frac{1}{\delta})+2\log(T^2{\left | \mathcal{U} \right |})}\]
We define the event
\[\mathcal{A}=\{\forall u\in \mathcal{U},\ \forall t_1<t_2,\ \left\|\bm{\Sigma}_{u,t_1,t_2}^{-1}\bm{g}_{u,t_1,t_2}\right\|_{\bm{\Sigma}_{u,t_1,t_2}}\leq \beta'(\delta)\}\]
Then the opposite event of $\mathcal{A}$ is
\[\mathcal{A}^c=\{\exists u\in \mathcal{U},\ \exists t_1<t_2,\ s.t.\left\|\bm{\Sigma}_{u,t_1,t_2}^{-1}\bm{g}_{u,t_1,t_2}\right\|_{\bm{\Sigma}_{u,t_1,t_2}}>\beta'(\delta)\}\]
With the union bound, the probability that event $\mathcal{A}$ happens can be bounded by
\begin{equation}\nonumber
\begin{aligned}
    \mathbf{P}(\mathcal{A})
    &=1-\mathbf{P}(\mathcal{A}^c)\\
    &\geq 1-\sum_{u=1}^{\left | \mathcal{U} \right |}\sum_{1\leq t_1<t_2\leq T}\mathbf{P}(\{\text{fix this agent}\ u\ \text{and the timestep}\ t_1,t_2,\ \left\|\bm{\Sigma}_{u,t_1,t_2}^{-1}\bm{g}_{u,t_1,t_2}\right\|_{\bm{\Sigma}_{u,t_1,t_2}}>\beta'(\delta)\}\})\\
    &\geq 1-{\left | \mathcal{U} \right |}\cdot T^2\cdot\frac{\delta}{T^2{\left | \mathcal{U} \right |}}\\
    &=1-\delta
\end{aligned}
\end{equation}
Finally, we let $t_1=N_u(t),t_2=t$, thus finishing the proof of this lemma.
\end{proof}
\cref{lemma 3.1}. (Local confidence bound, restated) Under our algorithm and setting, when the cluster is correct, then with probability at least $1-\delta$, for each $t \in [T]$ and every $u \in \mathcal{U}$ with respect to each cluster $\mathcal{V}_{j(u)}$, it holds that $\left\|\hat{\bm{\theta}}_{u,t}-\bm{\theta}^{j(u)}\right\|_{\bm{\Sigma}_{u,t}}\leq \beta(\delta)$. \par
\begin{proof}[\textit{Proof of \cref{lemma 3.1}}]
 Inspired by the proof of the Lemma B.1 in \citet{he2022simple} and observe that $\bm{\Sigma}_{u,t}$ and $\hat{\bm{\theta}}_{u,t}$ will not change if no communication happens, we only need to consider the case where the agent only communicates with the server at timesteps before $t_1$, i.e. $t \geq t_1+1$ and $N_u(t)=t_1$. If at $t=t_2$ communication happens, according to our algorithm, the agent will use the last communication information at $t_1$ to select action instead of this new one at $t_2$. For analysis, we focus on the end of each timestep $t$, i.e. $t+1$, so we take $t \geq t_1+1$. Remark that our analysis holds for both auxiliary communication and common communication. Then at the end of timestep $t$, the covariance matrix, and the feedback data are given by
\begin{equation}\label{eq.29}
\bm{\Sigma}_{u,t}=\bm{\Sigma}_{u,t_1}=\lambda I+\sum_{u'\in \mathcal{V}_{j(u)}}\bm{\Sigma}_{u', t_1}^{ser},\ \bm{b}_{u,t}=\bm{b}_{u,t_1}=\sum_{u'\in \mathcal{V}_{j(u)}}\bm{b}_{u', t_1}^{ser}
\end{equation}
For the estimated vector, we have
\begin{equation}\nonumber
\begin{aligned}
\hat{\bm{\theta}}_{u,t}
&=(\bm{\Sigma}_{u,t})^{-1}\bm{b}_{u,t}\\
&=\left( \lambda I+\sum_{u'\in \mathcal{V}_{j(u)}}\bm{\Sigma}_{u', t_1}^{ser}\right)^{-1}\left (\sum_{u'\in \mathcal{V}_{j(u)}}\bm{b}_{u', t_1}^{ser} \right )\\
&=\left( \lambda I+\sum_{u'\in \mathcal{V}_{j(u)}}\bm{\Sigma}_{u', t_1}^{ser}\right)^{-1}\left (\sum_{u'\in \mathcal{V}_{j(u)}}\bm{g}_{u', t_1}^{ser}+\sum_{u'\in \mathcal{V}_{j(u)}}\bm{\Sigma}_{u', t_1}^{ser}\bm{\theta}^{j(u)}\right )\\
&=\bm{\theta}^{j(u)}-\lambda\left( \lambda I+\sum_{u'\in \mathcal{V}_{j(u)}}\bm{\Sigma}_{u', t_1}^{ser}\right)^{-1}\bm{\theta}^{j(u)}+\left( \lambda I+\sum_{u'\in \mathcal{V}_{j(u)}}\bm{\Sigma}_{u', t_1}^{ser}\right)^{-1}\sum_{u'\in \mathcal{V}_{j(u)}}\bm{g}_{u', t_1}^{ser}\\
&=\bm{\theta}^{j(u)}-\lambda(\bm{\Sigma}_{u,t})^{-1}\bm{\theta}^{j(u)}+\sum_{u'\in \mathcal{V}_{j(u)}}(\bm{\Sigma}_{u,t})^{-1}\bm{g}_{u', t_1}^{ser}
\end{aligned}
\end{equation}
Then we have
\begin{equation}\nonumber
\begin{aligned}
\left\|\hat{\bm{\theta}}_{u,t}-\bm{\theta}^{j(u)}\right\|_{\bm{\Sigma}_{u,t}}
&\leq \lambda \left\|\bm{\theta}^{j(u)}\right\|_{(\bm{\Sigma}_{u,t})^{-1}}+ \left\|(\bm{\Sigma}_{u,t})^{-1}\sum_{u'\in \mathcal{V}_{j(u)}}\bm{g}_{u', t_1}^{ser}\right\|_{\bm{\Sigma}_{u,t}}\\
&\leq \sqrt{\lambda} \left\|\bm{\theta}^{j(u)}\right\|_2+ \left\|(\bm{\Sigma}_{u,t})^{-1}\sum_{u'\in \mathcal{V}_{j(u)}}\bm{g}_{u', t_1}^{ser}\right\|_{\bm{\Sigma}_{u,t}}\\
&\leq \sqrt{\lambda}S+\left\|(\bm{\Sigma}_{u,t})^{-1}\sum_{u'\in \mathcal{V}_{j(u)}}\bm{g}_{u', t_1}^{ser}\right\|_{\bm{\Sigma}_{u,t}}\\
&= \sqrt{\lambda}S+\left\|(\bm{\Sigma}_{u,t})^{-1}\sum_{u'\in \mathcal{V}_{j(u)}}\left(\bm{g}_{u', t_1}^{real}-\bm{g}_{u', t_1}^{loc} \right)\right\|_{\bm{\Sigma}_{u,t}}\\
&=\sqrt{\lambda}S+\left\|(\bm{\Sigma}_{u,t})^{-1}\bm{g}_{\mathcal{V}_{j(u)}, t_1}^{all}-(\bm{\Sigma}_{u,t})^{-1}\sum_{u'\in \mathcal{V}_{j(u)}}\bm{g}_{u', t_1}^{loc}\right\|_{\bm{\Sigma}_{u,t}}\\
&\leq \sqrt{\lambda}S+\underbrace{\left\|\bm{g}_{\mathcal{V}_{j(u)}, t_1}^{all}\right\|_{(\bm{\Sigma}_{u,t})^{-1}}}_A+\sum_{u'\in \mathcal{V}_{j(u)}}\underbrace{\left\|\bm{g}_{u', t_1}^{loc}\right\|_{(\bm{\Sigma}_{u,t})^{-1}}}_B\\
\end{aligned}
\end{equation}
The second inequality is because $\bm{\Sigma}_{u,t} \geq \lambda I$ and the third is because $\left\|\bm{\theta}^{j(u)}\right\|_2 \leq S$ ($S=1$ in our paper) for every cluster, then we bound the term (A) and (B) respectively. For term (A) with probability at least $1-\delta$:
\begin{equation}\label{eq.30}
\begin{aligned}
\left\|\bm{g}_{\mathcal{V}_{j(u)}, t_1}^{all}\right\|_{\bm{\Sigma}_{u,t}^{-1}} 
&\leq \sqrt{1+|\mathcal{V}_{j(u)}|\alpha_c}\left\|\bm{g}_{\mathcal{V}_{j(u)}, t_1}^{all}\right\|_{(\bm{\Sigma}_{\mathcal{V}_{j(u)},t_1}^{all})^{-1}}\\ 
&\leq \sqrt{1+|\mathcal{V}_{j(u)}|\alpha_c}R\sqrt{d\log(1+\frac{KT}{\lambda d})+2\log(\frac{1}{\delta})}\\
&\leq \sqrt{1+{\left | \mathcal{U} \right |}\alpha_c}R\sqrt{d\log(1+\frac{KT}{\lambda d})+2\log(\frac{1}{\delta})}
\end{aligned}
\end{equation}
The first inequality is from \cref{lemma 2.1} and the second is from \cref{lemma 1}.\par
For term (B), in our case, $N_u(t)=t_1$, so
\begin{equation}\label{eq.31}
\bm{\Sigma}_{u,t}=\bm{\Sigma}_{u,t_1}=\lambda I+\sum_{u'\in \mathcal{V}_{j(u)}}\bm{\Sigma}_{u', t_1}^{ser} \geq \frac{1}{\alpha_c}\bm{\Sigma}_{u, t_1}^{loc}
\end{equation}
This implies that
\begin{equation}\label{eq.32}
2\alpha_c\lambda I+\alpha_c\sum_{u'\in \mathcal{V}_{j(u)}}\bm{\Sigma}_{u', t_1}^{ser} \geq \alpha_c\lambda I+\bm{\Sigma}_{u, t_1}^{loc}
\end{equation}
\begin{equation}\label{eq.33}
\lambda I+\frac{1}{2}\sum_{u'\in \mathcal{V}_{j(u)}}\bm{\Sigma}_{u', t_1}^{ser} \geq \frac{1}{2\alpha_c}(\alpha_c\lambda I+\bm{\Sigma}_{u, t_1}^{loc})
\end{equation}
This further implies
\begin{equation}\label{eq.34}
\bm{\Sigma}_{u,t} \geq \frac{1}{2\alpha_c}(\alpha_c\lambda I+\bm{\Sigma}_{u, t_1}^{loc})
\end{equation}
With the above inequality and \cref{lemma 2.2}, with probability at least $1-\delta$, we can bound (B) as follows:
\begin{equation}\label{eq.35}
\begin{aligned}
\left\|\bm{g}_{u', t_1}^{loc}\right\|_{(\bm{\Sigma}_{u,t})^{-1}}\leq \sqrt{2\alpha_c}\left\|\bm{g}_{u', t_1}^{loc}\right\|_{(\alpha_c\lambda I+\bm{\Sigma}_{u, t_1}^{loc})^{-1}}\leq \sqrt{2\alpha_c}R\sqrt{d\log(1+\frac{KT}{\alpha_c \lambda d})+2\log(\frac{1}{\delta})+2\log(T^2{\left | \mathcal{U} \right |})}
\end{aligned}
\end{equation}
Combining the above inequalities, we have
\begin{equation}\label{eq.36}
\begin{aligned}
\left\|\hat{\bm{\theta}}_{u,t}-\bm{\theta}^{j(u)}\right\|_{\bm{\Sigma}_{u,t}}
&\leq \sqrt{\lambda}S+R(\sqrt{1+{\left | \mathcal{U} \right |}\alpha_c}+{\left | \mathcal{U} \right |}\sqrt{2\alpha_c})\sqrt{d\log(1+\frac{KT}{\min(1,\alpha_c) \lambda d})+2\log(\frac{1}{\delta})+4\log(T{\left | \mathcal{U} \right |})}
\end{aligned}
\end{equation}
And this inequality holds with probability at least $1-\delta/{\left | \mathcal{U} \right |}$, which finishes our proof.
\end{proof}
\subsection{Communication Cost Analysis}
\begin{proof}[\textit{Proof of Theorem 2}]
 Our communication cost consists of two parts: common communication and auxiliary communication. Similar to \citet{he2022simple}, the common communication can be bounded by
\begin{equation}\label{eq.37}
\begin{aligned}
2\log 2 \cdot d({\left | \mathcal{U} \right |}+1/\alpha_c)\log(1+KT/(\lambda d))
\end{aligned}
\end{equation}
For auxiliary communication, it can be bound in expectation by
\begin{equation}\label{eq.38}
\begin{aligned}
\sum_{t=1}^{T}p_t 
&\leq k\log T\sum_{t=1}^{T}\frac{1}{t} \\
&\leq k\log T(1+\sum_{t=2}^{T}\log(1+\frac{1}{t-1}))\\
&= k\log T(1+\sum_{t=1}^{T-1}\log(\frac{t+1}{t}))\\
&= k\log T(1+\log(T-1))\\
&\leq k(\log T+\log^2 T)
\end{aligned}
\end{equation}
Putting the above two inequalities together concludes Theorem 2.
\end{proof}
\section{Details for ForceComm Communication Protocol}
In this section, we propose another algorithm: ForceComm, which forces the agent to communicate with the server in its $1, 2, 4, 8, ..., 2^n$-th arrival to achieve sufficient communication rounds for user clustering. We outline the algorithm in \cref{alg:forcemain} and discuss its regret and communication cost in the following. We can see that ForceComm only changes \cref{alg:main} and the LocalAgent procedure (\cref{alg:local}).\par
\begin{algorithm}[htb] 
    \caption{ForceComm}\label{alg:forcemain} 
    \label{alg:Framwork} 
    \begin{algorithmic}[1] 
    \REQUIRE ~~\\ 
	Input: $\lambda > K$, $\alpha_c,\alpha_d > 0$\\
    Initialize server: 
    complete graph $G_0=(\mathcal{U}, E_0)$ over agents. $T_{u,0}^{ser}=0$, $\bm{b}_{u,0}^{ser}=\bm{b}_{u,0}^{clu}=\bm{0}_{d\times 1}$, $\hat{\bm{\theta}}_{u,0}^{ser}=\hat{\bm{\theta}}_{u,0}^{clu}=\bm{0}_{d\times 1}$, $\bm{\Sigma}_{u,0}^{ser}=\bm{\Sigma}_{u,0}^{clu}=\bm{0}_{d \times d}$, $count_{u,t}=0$\\
    Initialize local agent: $T_{u,0}^{loc}=0$, $\bm{b}_{u,0}=\bm{b}_{u,0}^{loc}=\bm{0}_{d\times 1}$, $\hat{\bm{\theta}}_{u,0}=\bm{0}_{d\times 1}$, $\bm{\Sigma}_{u,0}=\bm{\Sigma}_{u,0}^{loc}=\bm{0}_{d \times d}$\\
    \ENSURE ~~\\ 
    \FOR{round $t=1,...,T$}
    \STATE Agent $u_{t}$ is activated from an uniform distribution\label{line:forcemain_arrive}
	\STATE ComInd $\gets$\texttt{ForceComm-LocalAgent}($u_t, p_t, t, \alpha_c$)\label{line:forcemain_call_sub}
    \IF{ComInd$==$1}
    \STATE Agent $u_t$ sends $\bm{\Sigma}_{u_t,t}^{loc}$\label{line:forcemain_send}, $T_{u_t,t}^{loc}$ and $\bm{b}_{u_t,t}^{loc}$ to server 
	\STATE Run \texttt{Server}($u_t, t, \alpha_d$)
	\STATE Server sends $\bm{\Sigma}^{clu}_{u_t,t+1}$ and $\hat{\bm{\theta}}^{clu}_{u_t,t+1}$ back to agent $u_t$\label{line:forceserver_send}
 \STATE Agent $u_t$ update:  $\bm{\Sigma}^{loc}_{u_t,t+1}=\bm{0}_{d \times d}$, $\bm{b}^{loc}_{u_t,t+1}=\bm{0}_{d \times 1}$, $T_{u_t,t}^{loc}=0$, 
 $\bm{\Sigma}_{u_t,t+1}=\bm{\Sigma}^{clu}_{u_t,t+1}$, $\hat{\bm{\theta}}_{u_t,t+1}=\hat{\bm{\theta}}^{clu}_{u_t,t+1}$
    \ENDIF
    \ENDFOR 
    \end{algorithmic}
    \end{algorithm}
\begin{algorithm}[htb] 
    \caption{\texttt{ForceComm-LocalAgent}($u_t, p_t, t, \alpha_c$)} 
    \label{alg:forcelocal} 
    \begin{algorithmic}[1]
	\STATE  $count_{u_t,t}=count_{u_t,t-1}+1$
	\STATE  Receives context $\bm{x}_{t,i}, \forall i\in \mathcal{I}$\label{line:forcelocal_context}
	\STATE 
    $U_{t}(i) \gets\min\left\{\hat{\bm{\theta}}_{{u_t},t}^{T}\bm{x}_{t,i}+\beta(\delta)\left \| \bm{x}_{t,i} \right \|_{\bm{\Sigma}_{u_{t},t}^{-1}}, 1\right\}, \forall i\in \mathcal{I}$\label{line:forcelocal_ucb}
    \STATE $\bm{a}_{t} \gets oracle(U_{t}(1),\ldots,U_{t}(I))$\label{line:forcelocal_rank}
    \STATE Play $\bm{a}_{t}$ and observe $O_{t}, w_{t}(a_{k,t}),k \le O_{t}$ and receive reward $f(\bm{a}_{t},\bm{w}_{t})$
    \STATE $T_{u_t,t}^{loc}=T_{u_t,t-1}^{loc}+O_{t}$ \label{line:local_update_start}
    \STATE $\bm{\Sigma}_{u_{t},t}^{loc} = \bm{\Sigma}_{u_{t},t-1}^{loc}+\sum_{k=1}^{O_{t}}\bm{x}_{t,a_{k,t}}\bm{x}_{t,a_{k,t}}^{T}$, $\bm{b}_{u_{t},t}^{loc} = \bm{b}_{u_{t},t-1}^{loc}+\sum_{k=1}^{O_{t}}w_{t}(a_{k,t})\bm{x}_{t,a_{k,t}}$\label{line:forcelocal_update_end}
    \FOR{$u \ne u_t$}
    \STATE $T_{u,t}^{loc}=T_{u,t-1}^{loc}$, $count_{u,t}=count_{u,t-1}$
    \STATE $\bm{\Sigma}^{loc}_{u,t}=\bm{\Sigma}^{loc}_{u,t-1}$, $\bm{b}^{loc}_{u,t}=\bm{b}^{loc}_{u,t-1}$
    \STATE $\bm{\Sigma}_{u,t+1}=\bm{\Sigma}_{u,t}$, $\bm{b}_{u,t+1}=\bm{b}_{u,t}$, $\hat{\bm{\theta}}_{u,t+1}=\hat{\bm{\theta}}_{u,t}$
    \ENDFOR
    \IF{det($\bm{\Sigma}_{u_{t},t}+\bm{\Sigma}_{u_t,t}^{loc}$) $>(1+\alpha_c)$det($\bm{\Sigma}_{u_{t},t}$) or there exist $n \in \mathbb{N}_+$ such that $count_{u_t,t}=2^n$} \label{line:forcelocal_com_det}
    \STATE {\bf Return:} 1
    \ELSE
    \STATE $\bm{\Sigma}_{u_t,t+1}=\bm{\Sigma}_{u_t,t}$, $\bm{b}_{u_t,t+1}=\bm{b}_{u_t,t}$, $\hat{\bm{\theta}}_{u_t,t+1}=\hat{\bm{\theta}}_{u_t,t}$
    \STATE $G_{t+1}=G_t$
    \STATE {\bf Return:} 0
    \ENDIF
    \end{algorithmic}
    \end{algorithm}
The analysis of ForceComm is different from that of FedC$^3$UCB-H only in the quantity of $T_1$. For each agent $u$ we define that:
\[
D_u=\{u \text{ communicated with server at least once during } T_0\sim T_0+T_1\}
\]
\[
D_{all}=\{\text{all agent communicated with server at least once during }T_0\sim T_0+T_1\}
\]
Firstly, we can conclude that for any $T>0$, it holds that $2^{\left \lfloor log_2 T \right \rfloor}\leq T < 2^{\left \lfloor log_2 T \right \rfloor+1}\leq 2T$, which means that for any $T>0$, there always exist a positive integer $k$ such that $2^k$ contained in $T\sim 2T$.\par
According to the communication of ForceComm, we have:
\begin{equation}\nonumber
\begin{aligned}
P(D_u)
&=P(\text{there exist }k, \text{ such that }count_{u,T_0+T_1}\geq 2^k > count_{u,T_0})\\
&=P(count_{u,T_0+T_1} \geq 2^{\left \lfloor log_2 count_{u,T_0} \right \rfloor +1}>count_{u,T_0}\geq 2^{\left \lfloor log_2 count_{u,T_0} \right \rfloor})\\
&=P(count_{u,T_0+T_1}-count_{u,T_0}\geq 2^{\left \lfloor log_2 count_{u,T_0} \right \rfloor})\\
&=P(\text{the activate time of agent }u\geq 2^{\left \lfloor log_2 count_{u,T_0} \right \rfloor} \text{ during }T_0\sim T_0+T_1)\\
&\geq P(\text{the activate time of agent }u\geq count_{u,T_0} \text{ during }T_0\sim T_0+T_1)
\end{aligned}
\end{equation}
The last inequality is because $count_{u,T_0}\geq 2^{\left \lfloor log_2 count_{u,T_0} \right \rfloor}$.\par
Since $\sum_{u}count_{u,T_0}=T_0$. Because $count_{u,T_0}<T_0$, we have
\[
P(D_u^c)\leq P(\text{the activate time of agent }u \text{ less than } T_0 \text{ during }T_0\sim T_0+T_1)
\]
So we have
\begin{equation}\nonumber
\begin{aligned}
P(D_{all}^c)\leq \sum_{u}P(D_u^c)\leq \sum_{u}P(\text{the activate time of agent }u\text{ less than } T_0 \text{ during }T_0\sim T_0+T_1)
\end{aligned}
\end{equation}
According to Lemma 8. in \cite{li2018online}, if we take $T_1\geq 16{\left | \mathcal{U} \right |}\log{\frac{4T}{\delta}}+4{\left | \mathcal{U} \right |}T_0$ then we have
\[
P(D_{all})\geq P(\text{the activated number of every agents greater than }T_0 \text{ during }T_0\sim T_0+T_1)>1-\frac{\delta}{4}
\]
According to $T_0$, we have
\[
T_1=O({\left | \mathcal{U} \right |}^2\log{T/\delta})
\]
Replacing this with the definition of $T_1$ of FedC$^3$UCB-H and let $\delta={\left | \mathcal{U} \right |}^2/T$, we achieve the cumulative regret of ForceComm:
\begin{equation}\nonumber
\begin{split}
2dK\left | \mathcal{U} \right |\log(1+ KT/\lambda)+4\beta(\delta)\sqrt{J+\left | \mathcal{U} \right |\alpha_c}\sqrt{dKT\log(1+KT/\lambda)}\\+O(\left | \mathcal{U} \right |(\frac{d}{\gamma^2\tilde{\lambda}_x}+\frac{1}{\tilde{\lambda}^2_x})\log{T})
\end{split}
\end{equation}
Similarly, the communication cost can be bound by:
\[
2\log 2 \cdot d({\left | \mathcal{U} \right |}+1/\alpha_c)\log(1+KT/(\lambda d))+\left | \mathcal{U} \right |\log_2{T}
\]
\end{document}